\documentclass[final]{siamltex}

\pdfoutput=1

\usepackage{subcaption}
\usepackage{hyperref} %BUG if put after
\usepackage[latin1]{inputenc}
\usepackage{mystyle}
\usepackage{notations}
\usepackage{url}

\usepackage{algorithm}
\usepackage{algorithmic}

\graphicspath{{./images-low/}}

\usepackage[square]{natbib}
\setcitestyle{numbers}

\hypersetup{  
   bookmarks=true,
   backref=true,
   pagebackref=false,
   colorlinks=true,
   linkcolor=blue,
   citecolor=red,
   urlcolor=blue,
   pdftitle={Semi-dual Regularized Optimal Transport},
   pdfauthor={Marco Cuturi and Gabriel Peyr\'e},
   pdfsubject={Semi-dual Regularized Optimal Transport}
}

\title{Semi-dual Regularized Optimal Transport}

\author{Marco Cuturi\thanks{\url{marco.cuturi@ensae.fr}, CREST, ENSAE}
\and 
Gabriel Peyr\'e\thanks{\url{gabriel.peyre@ens.fr}, CNRS and DMA, \'Ecole Normale Sup\'erieure, FRANCE}
}

\begin{document}
%siam_id=M103260
%CODEN=SJISBI

\maketitle

% \newcommand{\slugmaster}{%
% \slugger{siims}{2016}{9}{1}{320--343}}

% !TEX root = ../SIGEST-SemiDual.tex

% \setcounter{page}{320}
\begin{abstract}
Variational problems that involve Wasserstein distances and more generally optimal transport (OT) theory are playing an increasingly important role in data sciences. Such problems can be used to form an examplar measure out of various probability measures, as in the Wasserstein barycenter problem, or to carry out parametric inference and density fitting, where the loss is measured in terms of an optimal transport cost to the measure of observations. Despite being conceptually simple, such problems are computationally challenging because they involve minimizing over quantities (Wasserstein distances) that are themselves hard to compute.
Entropic regularization has recently emerged as an efficient tool to approximate the solution of such variational Wasserstein problems. In this paper, we give a thorough duality tour of these regularization techniques. In particular, we show how important concepts from classical OT such as $c$-transforms and semi-discrete approaches translate into similar ideas in a regularized setting. 
These dual formulations lead to smooth variational problems, which can be solved using smooth, differentiable and convex optimization problems that are simpler to implement and numerically more stable that their un-regularized counterparts.
We illustrate the versatility of this approach by applying it to the computation of Wasserstein barycenters and gradient flows of spatial regularization functionals.
% We show that the dual formulation of Wasserstein variational problems introduced recently by G.~Carlier, A.~Oberman, and E.~Oudet [{\em ESAIM Math.\ Model.\ Numer.\ Anal.}, 6 (2015), pp.~1621--1642]  can be regularized using an entropic smoothing, which leads to smooth, differentiable and convex optimization problems that are simpler to implement and numerically more stable. We illustrate the versatility of this approach by applying it to the computation of Wasserstein barycenters and gradient flows of spatial regularization functionals.
\end{abstract}



% \begin{DOI}
% 10.1137/15M1032600
% \end{DOI}

% MOD
\pagestyle{myheadings}
\thispagestyle{plain}
\markboth{MARCO CUTURI AND \smash[t]{GABRIEL PEYR\'E}}{Semi-dual Regularized Optimal Transport}

% !TEX root = ../SIGEST-SemiDual.tex

\section{Introduction}

% Optimal transport is a well established research subject at the interplay between mathematical analysis and probability~\cite{Villani03}. Optimal transport provides a rich set of tools to define a geometry for probability measures supported on a set which is itself a metric space, among which the family of Wasserstein distances (\emph{a.k.a} earth mover's distances,~\citealt{RubTomGui00}). The goal of the present work is to define tractable optimization schemes to deal with \emph{convex variational problems} involving transportation distances, the most representative of them being the computation of the barycenter of a set of histograms under the Wasserstein metric.

% Histograms are fundamental objects popularly used in machine learning to represent complex objects as frequency vectors in the probability simplex. By defining first a set of relevant features, one can then form for each object a normalized histogram that keeps track of the frequencies of each of these features---\eg, bags-of-words for text \cite{salton1975vector}, bags-of-visual-words for images \cite{Lowe1999,oliva2001modeling}. 

To compare two probability vectors (or histograms) in the probability simplex, information divergences---the Hellinger and $\chi_2$ distances, the Kullback--Leibler and Jensen-Shannon diver\-gen\-ces---proceed by comparing them elementwise. Although this feature is attractive from a computational point-of-view, one of its drawbacks in high-dimensional regimes is that histograms tend to be sparse and share little common overlap when considered in pairs, making information divergences less informative. Optimal transport distances~\cite[section~7]{villani09}---a.k.a.\ Wasserstein or earth mover's distances~\cite{RubTomGui00}---are usually seen as a remedy to this issue. This is because they can incorporate a substitution cost between the bins of these histograms to account for the lack of overlap in their support: Intuitively, it does not matter that two histograms carry mass on differing bins, as long as those bins are similar, \ie have low substitution cost.  Because of this flexibility, Wasserstein distances have found popular applications in data sciences, such as for instance image retrieval~\cite{RubTomGui00,Pele-iccv2009}, image interpolation~\cite{Bonneel-displacement}, computational geometry~\cite{Merigot2011}, color image processing~\cite{2013-Bonneel-barycenter,2014-xia-siims}, image registration~\cite{MassGenTransport}, or machine learning~\cite{CourtyFlamary}.

The versatility of optimal transport distances comes, however, at a price: evaluating the OT distance between two histograms requires computing an optimal coupling, that is a matrix with as many lines and columns than those histograms which is well aligned with the substitution cost matrix. Finding that optimal coupling is typically done by solving a network flow problem, whose cost scales supercubicly with the dimension of the considered histograms. That cost becomes even more of a drawback if one attempts to study so called "variational Wasserstein problems" that consists in minimizing an objective function which defined through various OT distances between several histograms.

%%%%%%%%%%%%%%%%%%%%%%%%%%%%%%%%%%%%%%
%SEC 1.1
\subsection{Variational Wasserstein problems} 

Many learning tasks on probability vectors, such as averaging or clustering, can be framed as variational problems that involve distances between pairs of measures. These problems are easily solved when such divergences are Bregman divergences~\cite{lee1999learning,banerjee2005clustering,JeffreysCentroid-2013}, but  they are far more challenging when considering instead Wasserstein distances. \citeauthor{Carlier_wasserstein_barycenter} \cite{Carlier_wasserstein_barycenter} studied the first problem of this type, the Wasserstein barycenter problem (WBP), and showed that it is related to the multi-marginal optimal transport problem. More recently,~\cite{Solomon-ICML} proposed the Wasserstein propagation-on-graphs framework, showing that it can be solved through a large linear program. Other recent developments include clustering with weight constraints~\citep{cuturi2014fast}, statistics to develop population estimators~\cite{BigotBarycenter} or merge probabilities in a Bayesian setting~\cite{JMLR:v18:16-655}, computer graphics to perform image modification~\cite{2014-xia-siims,2013-Bonneel-barycenter,2015-solomon-siggraph}, computer vision~\cite{ZenICPR14} to summarize complex visual signals.

Another important example of Wasserstein variational problems is the approximation of gradient flows. As initially shown by~\cite{jordan1998variational}, it is indeed possible to approximate solutions of a large family of partial differential equations by iteratively minimizing some energy functional plus a Wasserstein distance to the previous iterate. We refer to section~\ref{sec-grad-flows} for more details and references about these schemes. 

Beside the computation of barycenters, it is also possible to integrate Wasserstein distances into more general variational problems. 
For instance, optimal transport distances are used as a data fidelity to perform image denoising~\citep{Burger-JKO,Lorentz}, image segmentation~\citep{RabinPapadakisSSVM,SchnorSegmentation,SchmitzerSegmentation} and Radon transform reconstruction~\citep{AbrahamRadon,BenamouCCNP15}.

Our aim in this paper is to propose a computational framework that is both scalable and flexible enough to minimize energies that involve not only Wasserstein distances but also more general functions such as regularization terms~\citep{10.1007/978-3-319-68445-1_10}. To do so, we exploit regularization, Legendre duality, and the usual toolbox of convex optimization.

%%%%%%%%%%%%%%%%%%%%%%%%%%%%%%%%%%%%%%
%SEC 1.2
\subsection{Previous works}
\label{sec-pw}

%%%
\paragraph{Duality and Semi-discrete OT}

Duality plays a key role in both the theoretical analysis and the numerical resolution of optimal transport. Indeed, most standard solvers, such as minimum cost network flows or auction algorithms (see~\citep[\S3]{2018-Peyre-computational-ot} for a review), build upon duality. 
The idea of duality finds a particularly powerful formulation when solving the OT problem not only between histograms or discrete measures, but more generally between a continuous measure and a discrete one. This problem is known as the ``semi-discrete'' OT problem, which was first investigated by~\cite{oliker1989numerical} and later refined in~\cite{AurenhammerHA98}. They use the machinery of $c$-transforms to convert the resolution of the OT between a continuous density and a discrete measure by turning it into a convex \emph{finite} dimensional optimization problem. 
The recent revival of this method is due to~\cite{Merigot11} who proposed a quasi-Newton solver and clarified the link with concepts from computational geometry. We refer to~\cite{Levy2017review} for a recent overview. The use of a Newton solver which is applied to sampling in computer graphics is proposed in~\cite{de2012blue}, see also~\cite{levy2015numerical} for applications to 3-D volume and surface processing. 
A similar ``semi-dual'' formulation was used in~\cite{Carlier-NumericsBarycenters} to tackle the Wasserstein barycenter problem. One of the goals of our paper is to extend this formulation to include an entropic regularization and thus obtain a smooth optimization problem.  
As shown in~\cite{genevay2016stochastic}, another interest of this semi-dual formulation is that it enables the use of stochastic optimization methods, to tackle ``black box'' settings where the continuous density can only be accessed through i.i.d. sampling. This idea has been extended in~\cite{claici2018stochastic} for the computation of barycenters. 

%%%
\paragraph{Entropic regularization}

Despite being an old idea (for instance it is related to {Schr{\"o}dinger}'s problem~\cite{leonard2013survey}), the use of an entropic regularization to approximate the solution of OT-related problem has recently gained a lot of momentum in the imaging and machine learning community~\cite{cuturi2013sinkhorn}. 
% (see also~\cite{galichon2010matching} for applications in social sciences).
%
This regularization can be solved using Sinkhorn's algorihtm~\cite{Sinkhorn64}, which leads to a parallelizable numerical scheme which streams well on modern GPUs' architectures.
Beside its cheap computational cost, the chief advantage of entropic regularized OT is that it leads to a smooth transportation cost. It thus enables the use of simple optimization schemes for various Wasserstein variational problems, such as gradient descent in~\cite{cuturi2014fast} to solve the WBP; a more elaborate generalization of Sinkhorn's iterations in~\cite{BenamouCCNP15} tailored for the WBP, see also~\cite{solomon2015convolutional,bonneel2016wasserstein} for applications in image processing and computer graphics; Wasserstein gradient flows~\cite{Peyre-JKO}.

One of the goal of this paper is to bridge the gap between entropic regularization and semi-discrete methods, and to combine the strength of both approaches:  smooth convex optimization schemes that can be applied with full generality, without having to resort to Sinkhorn iterations (either as an inner approximation loop or as a generalized projection~\cite{BenamouCCNP15}).

% Their formulation requires, however, running a numerical subroutine, the Sinkhorn fixed-point iteration, to evaluate these objectives and compute their gradients. On the other hand, \cite{Carlier-NumericsBarycenters} show that the Fenchel--Legendre dual of the Wasserstein distance as well as its subgradients can be obtained in \emph{closed form} using nearest-neighbor assignments, that is, without having to solve a single optimal transport problem. The authors do, however, struggle with nondifferentiable objective functions and use a L-BFGS first order scheme. More recently, \cite{BenamouCCNP15} have proposed a generalized version of Sinkhorn's algorithm to compute barycenters based on Bregman's projections. This approach is useful for the barycenter problem but cannot be easily adapted to solve more involved problems that incorporate terms, such as regularizers, that are not Wasserstein distances.

% Note that this application requires computing the gradient of the dual of the smoothed Wasserstein distance with respect to two histograms. This formula is provided in Appendix~\ref{sub:two}. 

% \todo{Explain limits by adding reference to Entropic Wasserstein Gradient Flows? we should definitely add that reference}

%%%%%%%%%%%%%%%%%%%%%%%%%%%%%%%%%%%%%%
%SEC 1.3
\subsection{Contributions}

The main contribution of this paper is the extension of the classical ``semi-discrete'' formulation of OT to the setting of entropic regularization, to yield an OT problem that we call a ``semi-dual'' approach.
The content proposed here is a revised and extended version of the original paper~\cite{2016-Cuturi-siims}. 
We first detail in Section~\ref{sec:duals} the semi-dual formulation in the simple case of two input measures. As a by-product, we give formula for the Legendre transform of the regularized OT cost together with their gradient and Hessian, which is crucial to use efficient (quasi)-Newton solvers.
We show how this semi-dual formulation can be applied to various variational problems including the computation of barycenters in Section~\ref{sec:dualalgo} and more advanced problems (barycenters with extra regularity constraints and gradient flows) in Section~\ref{sec:exten}. 
Another illustration of the usefulness of this semi-dual approach is the application to image segmentation developed in~\cite{RabinPapadakisSSVM}.

% Our main contribution is to combine the strengths of the dual formulation of~\cite{Carlier-NumericsBarycenters} with the smoothing strategy laid out by~\cite{cuturi2014fast} to obtain a \emph{smooth} optimization problem whose objectives and derivatives can be computed in \emph{closed form} in section~\ref{sec:optimtransentrop}. We show that this approach can be readily used to compute Wasserstein barycenters in section~\ref{sec:dualalgo} and explain why using regularized Wasserstein distances might be beneficial to recover smooth solutions. 
% We proceed with more general energies that involve not only Wasserstein distances but also more generally spatial regularization of barycenters and gradient flows, in section~\ref{sec:exten}. 

The source code to reproduce the numerical illustrations of this article can be found online.\footnote{See \href{https://github.com/gpeyre/2015-SIIMS-wasserstein-dual/}{https://github.com/gpeyre/2015-SIIMS-wasserstein-dual/}.}

%\item Third, we show in Sections~\ref{sec:exten} that this approach is versatile and can serve as a blueprint to solve more advanced Wasserstein variational problem, such as Wasserstein propagation, as well as penalized or constrained barycenter problems. We provide a testing ground for these ideas by studying in depth the problem of averaging \emph{unnormalized measures} under a suitable extension of the Wasserstein metric. Indeed, in some application fields, the ability to take into account not only the distribution of various measures, but also their relative mass, can be crucial. We propose an algorithmic answer to this problem and illustrate it by addressing the hard problem of averaging brain activation maps defined on the folded triangulated cortex, \ie, producing a mean brain activation in a group of subjects. We present results both on realistic simulations as well as on publicly available magneto-encephalography (MEG) data.

%%%%%%%%%%%%%%%%%%%%%%%%%%%%%%%%%%%%%%
%SEC 1.4
\subsection{Notation}

When used on matrices, functions such as $\log$ or $\exp$ are always applied elementwise. For two matrices (or vectors) $\mathbf{A},\mathbf{B}$ of the same size, $\mathbf{A}\circ \mathbf{B}$ (resp., $\mathbf{A}/\mathbf{B}$) stands for the elementwise product (resp., division) of $\mathbf{A}$ by $\mathbf{B}$. If $\mathbf{u}$ is a vector, $\diag(\mathbf{u})$ is the diagonal matrix with diagonal $\mathbf{u}$. $\ones_{\n} \in \RR^{\n}$ is the (column) vector of ones.
Given an integer $n\geq 1$, we write $\Si_n$ for the discrete probability simplex
\eq{ 
	\Si_n \defeq \enscond{\a \in \RR_+^n}{\sum_i \a_i=1}.
}
The discrete entropy of a matrix $\P$ is defined as
\eql{\label{eq-discr-entropy}
	\HD(\P) \eqdef -\sum_{i,j} \P_{i,j} (\log(\P_{i,j})-1), 
}
with the convention $0 \log(0)=0$, and $\HD(\P) = -\infty$ if one of the entries $\P_{i,j}$ is negative.

% !TEX root = ../SIGEST-SemiDual.tex

%SEC 2
%\input sections/sec-wasserstein
\section{Duals of Regularized Optimal Transport}\label{sec:duals}

We introduce in this section the entropic regularization of the OT problem, study several dual formulations as well as their associated Legendre transforms, to show that they admit a simple closed form. 

%%%%%%%%%%%%%%%%%%%%%%%%%%%%%%%%%%%%%%%%%%%%%%%%%%%%%%%%%%%%%%%%%%%%%%%%%%%%%%%
\subsection{Regularized Optimal Transport}

We consider in what follows two histograms $(\a,\b) \in \simplex_n \times \simplex_m$.
These histograms account for the weights of two probability measures supported on two discrete sets of respective sizes $n,m$. This notation is slightly more general than the one used in~\cite{2016-Cuturi-siims}, since it allows us to consider measures which are not necessarily supported on the same points/bins.

%  In many applications of optimal transport, the cost matrix $\C$ is defined through points $(x_i)_i$ and $(y_j)_j$ taken in a metric space $(\Xcal,D)$ such that $\C_{i,j} = D(x_i,x_j)^\rho, \rho\geq 1$. Note, however, that we make no assumption on $\C$ in this paper other than the fact that it is symmetric and nonnegative.

The set of couplings linking a pair of histograms $(\a,\b)$ is defined as
\eq{
	\CouplingsD(\a,\b) \defeq \enscond{\P \in \RR_+^{\n \times m}}{ \P \ones_m = \a, \P^\top \ones_n = \b }.
}
Here $\P_{i,j}$ represent the amount of mass transferred from bin $i$ to bin $j$, and the constraints account for the conservation of mass. The set $\CouplingsD(\a,\b)$ is bounded in $\RR^{\n\times m}$, defined by linear equalities and non-negativity constraints, and is therefore a convex polytope.

To define optimal transport, we consider a cost matrix $\C\in \RR^{\n \times m}$. Each element $\C_{i,j}$ accounts for a substitution cost between $i$ and $j$, or equivalently as a (ground) cost required to move a unit of mass from bin $i$ to bin $j$.
The entropic regularization of the Kantorovitch formulation of OT mentioned first in Section~\ref{sec-pw} reads 
\eql{\label{eq-regularized-discr}
	\MKD_\C^\epsilon(\a,\b) \eqdef 
	\umin{\P \in \CouplingsD(\a,\b)}
		\dotp{\P}{\C} - \epsilon \HD(\P), 
} 
where $\epsilon\geq 0$, and where the inner product is defined as
\eq{
	\dotp{\C}{\P} \defeq \sum_{i,j} \C_{i,j} \P_{i,j}.
}
The case $\epsilon=0$ corresponds to the usual (linear) optimal transport problem. For $\epsilon>0$, Problem~\eqref{eq-regularized-discr} has an $\epsilon$-strongly convex objective and admits therefore a unique optimal solution $\P_\epsilon^\star$. 
While $\P_\epsilon^\star$ is not necessarily unique for $\epsilon=0$, we show in the following proposition that in the small $\epsilon$ limit, the regularization aims for the coupling, among those that are optimal for the linear problem, that has the highest entropy.

\begin{proposition}[Convergence with $\epsilon$]\label{prop-convergence-entropy}
The unique solution $\P_\epsilon$ of~\eqref{eq-regularized-discr} converges to the optimal solution with maximal entropy within the set of all optimal solutions of the Kantorovich problem, namely
\eql{\label{eq-entropy-conv-1}
	\P_\epsilon \overset{\epsilon \rightarrow 0}{\longrightarrow}
	\uargmin{\P} \enscond{ -\HD(\P) }{
		\P \in \CouplingsD(\a,\b), \dotp{\P}{\C} = \MKD_\C^0(\a,\b)
	}
}
so that in particular $\MKD_\C^\epsilon(\a,\b) \overset{\epsilon \rightarrow 0}{\longrightarrow} \MKD_\C^0(\a,\b)$.
One also has
\eql{\label{eq-entropy-conv-2}
	\P_\epsilon \overset{\epsilon \rightarrow \infty}{\longrightarrow}
	\a \transp{\b} = (\a_i \b_j)_{i,j}.
}
\end{proposition}
\begin{proof}
 	We consider a sequence $(\epsilon_\ell)_\ell$ such that $\epsilon_\ell \rightarrow 0$ and $\epsilon_\ell > 0$.	
 	We denote $\P_\ell = \P^\star_{\epsilon_\ell}$. Since $\CouplingsD(\a,\b)$ is bounded, we can extract a sequence (that we do not relabel for sake of simplicity) such that $\P_\ell \rightarrow \P^\star$. Since $\CouplingsD(\a,\b)$ is closed, $\P^\star \in \CouplingsD(\a,\b)$. We consider any $\P$ such that $\dotp{\C}{\P} = W_0(\a,\b)$. By optimality of $\P$ and $\P_\ell$ for their respective optimization problems (for $\epsilon=0$ and $\epsilon=\epsilon_\ell$), one has
 	\eql{\label{eq-proof-gamma-conv-proof}
 		0 \leq \dotp{\C}{\P_\ell} - \dotp{\C}{\P} \leq \epsilon_\ell ( \HD(\P_\ell)-\HD(\P) ).
 	}
 	Since $\HD$ is continuous, taking the limit $\ell \rightarrow +\infty$ in this expression shows that 
 	$\dotp{\C}{\P^\star} = \dotp{\C}{\P}$ so that $\P^\star$ is a feasible point of the right hand side appearing in~\eqref{eq-entropy-conv-1}. Furthermore, dividing by $\epsilon_\ell$ in~\eqref{eq-proof-gamma-conv-proof} and taking the limit shows that 
 	$\HD(\P) \leq \HD(\P^\star)$, which shows that $\P^\star$ is a solution of~\eqref{eq-entropy-conv-1}. Since the solution $\P_0^\star$ to this program is unique by strict convexity of $-\HD$, one has $\P^\star = \P_0^\star$, and the whole sequence is converging.\qquad
\end{proof}

A detailed analysis of the convergence with $\epsilon$ (in particular a first order expansion of the trajectory $\epsilon \mapsto \P_\epsilon$) can be found in~\cite{CominettiAsympt}.

%%%%%%%%%%%%%%%%%%%%%%%%%%%%%%%%%%%%%%%%%%%%%%%%%%%%%%%%%%%%%%%%%%%%%%%%%%%%%%%
\subsection{Duals and Semi-duals}
\label{sec-duals-semi-duals}

Duality plays a crucial role in the theoretical analysis of optimal transport, notably when studying a continuous formulation of the OT problem \eqref{eq-regularized-discr} that involves Radon measures. In that context, duality can be used to prove existence of solutions, and most importantly to show that in some cases optimal couplings are degenerate and correspond to so-called Monge transportation maps~\cite[\S1.1]{santambrogio2015optimal}. These theoretical considerations aside, duality is also pivotal to derive efficient primal-dual linear solvers. It is therefore not surprising that these considerations extend to the regularized problem~\eqref{eq-regularized-discr} when $\epsilon>0$.

%%%%
\paragraph{Dual with respect to both inputs}

We first detail below two dual formulations: the first one~\eqref{eq-dual-formulation} extends the classical OT dual to the regularized setting, while the second one~\eqref{eq-dual-formulation-bach} was suggested to us by Francis Bach. 

\begin{prop}
One has that for $\epsilon \geq 0$ the two equivalent formulations
\begin{align}\label{eq-dual-formulation}
	\MKD_\C^\epsilon(\a,\b) &= \umax{\fD \in \RR^n,\gD \in \RR^m}
		 %\Dd_\epsilon(\fD,\gD) \eqdef 
		 \dotp{\fD}{\a} + \dotp{\gD}{\b} 
		+ B_\epsilon(\C - \fD \oplus \gD)  \\
		& \label{eq-dual-formulation-bach} = \umax{\fD \in \RR^n,\gD \in \RR^m}
		 % \bar\Dd_\epsilon(\fD,\gD) \eqdef 
		 \dotp{\fD}{\a} + \dotp{\gD}{\b} 
		+ {\min}_\epsilon(\C - \fD \oplus \gD) 
\end{align} 
Here we denoted for $\SD \in \RR^{n \times m}$
\begin{align*}
	B_0(\SD) &\eqdef -\iota_{\RR_{+}^{n \times m}}(\SD)
	\qandq
	{\min}_0(\SD) \eqdef \min_{i,j} \SD_{i,j}, \\
	\foralls \epsilon>0, \quad
	B_\epsilon(\SD) &\eqdef - \epsilon \sum_{i,j} e^{-\SD_{i,j}/\epsilon}
	\qandq
	{\min}_\epsilon(\SD) \eqdef - \epsilon \log \sum_{i,j} e^{-\SD_{i,j}/\epsilon}.
\end{align*}
\end{prop}

\begin{proof}
	Formula~\eqref{eq-dual-formulation} follows from the usual derivation, introducing dual variables $(\fD,\gD)$ associated to the constraints $(\P \ones_m, \P^\top \ones_n) = (\a,\b)$. 
	Formula~\eqref{eq-dual-formulation-bach} is obtained by introducing the redundant constraint $\sum_{i,j} \P_{i,j}=1$, and thus an extra dual variable associated to this constraint.
\end{proof}

These formulas show that $\MKD_\C^\epsilon(\a,\b)$ is a convex function of $(\a,\b)$ (as a maximum of affine functions). 
An advantage of~\eqref{eq-dual-formulation-bach} over~\eqref{eq-dual-formulation} is that it corresponds to the minimization of a uniformly Lipschitz function that is unconstrained even when $\epsilon=0$.
As we detail in Corollary~\ref{cor-legendre-both} below, in the case $\epsilon>0$, this alternative dual is also better because it has a bounded Hessian, and can thus be minimized by efficient solvers. 

As detailed in the following proposition, the second dual formula~\eqref{eq-dual-formulation-bach} can equivalently be stated as the explicit expression of the Fenchel--Legendre transform of the optimal transportation cost, defined as:
\eql{\label{eq-dfn-legendre-both}
	\foralls (\fD,\gD) \in \RR^n \times \RR^m, \quad
	\MKD_\C^{\epsilon,*}(\fD,\gD) \defeq \umax{(\a,\b) \in \Si_n \times \Si_m} \dotp{\fD}{\a} + \dotp{\gD}{\b}  - \MKD_\C^\epsilon(\a,\b).
}
It is important to note that for the transform~\eqref{eq-dfn-legendre-both}, we restrict $\MKD_\C^\epsilon$ to be defined on probability histograms in $\Si_n \times \Si_m$. One can also consider $\MKD_\C^\epsilon$ as being defined on $\RR_+^n \times \RR_+^m$ with the additional constraint that both masses are equal $\sum_i \a_i = \sum_j \b_j$. In this case, one needs to consider the ``classical'' dual~\eqref{eq-dual-formulation} in place of~\eqref{eq-dual-formulation-bach}.  

% \todo{Check if we put this in appendix. I think the formula for the Hessian could be nicely factorized and simplified. }

\begin{cor}\label{cor-legendre-both}
One has for $\epsilon \geq 0$
\eql{\label{eq-legendre-both}
	\foralls (\fD,\gD) \in \RR^n \times \RR^m, \quad
	\MKD_\C^{\epsilon,*}(\fD,\gD) = - {\min}_\epsilon(\C - \fD \oplus \gD). 
}
For $\epsilon>0$, $\MKD_\C^{\epsilon,*}$ is $C^\infty$ and its gradient function $\nabla \MKD_\C^{\epsilon,*}(\cdot)$ is $2/\epsilon$ Lipschitz. 
%Its value, gradient, and Hessian at $(\fD,\gD) \in \RR^n \times \RR^m$ are, writing $\uD \eqdef e^{\fD/\epsilon}$, $\vD \eqdef e^{\gD/\epsilon}$, $\K \eqdef e^{-\C/\epsilon}$, $\P \eqdef \diag(\uD) \K \diag(\vD)$, 
%$\bar \a = \uD \odot (\K \vD)$, 
%$\bar\b = \vD \odot (\K^\top \uD)$, \todo{$\ga=1/\epsilon$?}
%\begin{gather*}% \label{eq-obj-bothvar}
%\begin{split}
%  		\MKD_\C^{\epsilon,*}(\fD,\gD) &= -\epsilon \log \dotp{\uD}{\K \vD},\\ 
%		\nabla \MKD_\C^{\epsilon,*}(\fD,\gD) &= \frac{1}{\uD^T \K \vD}
%			\begin{bmatrix} \bar\a \\ \bar\b \end{bmatrix},\\
%		\nabla^2 \MKD_\C^{\epsilon,*}(\fD,\gD) &= \frac{1}{\ga\uD^T K \vD} 
%		\begin{bmatrix} 
%			A_\epsilon(\fD,\gD) & B_\epsilon(\fD,\gD) \\ 
%			B_\epsilon(\gD,\fD) & A_\epsilon(\gD,\fD)
%		\end{bmatrix},				
%\end{split}
%\\
%	\qwhereq					
%	\begin{cases}
%		A_\epsilon(\fD,\gD) =  \diag(\bar\a)- \frac{1}{\uD^\top \K \vD} \bar\a \bar\b^\top,\\
%		B_\epsilon(\fD,\gD) = \P - \frac{1}{\uD^\top \K \vD} \bar\b \bar\a^\top.
%	\end{cases}
%\end{gather*}
\end{cor}

%\begin{proof}
%	\todo{write me}
%	The Lipschitz continuity of the gradient can be obtained by showing that the Hessian's trace can be upper-bounded by $2/\epsilon$ by noticing that the trace of both $A_\ga(\fD,\gD)$ and $A_\ga(\gD,\fD)$ is upper-bounded by $\uD^T (\K\vD)$.
%\end{proof}

%%%%
\paragraph{$\c$-transforms and Sinkhorn}

A useful generalization of convexity used in optimal transport theory comes from so-called $\c$-transforms \cite[\S1.2]{santambrogio2015optimal}. $\c$-transforms arise naturally when minimizing explicitly the dual OT problem with respect to only one of its two variables, while keeping the other fixed. That tool is important from a theoretical perspective to prove the regularity of the dual variables for generic measures (see Section~\ref{sec-gen-measure}) and also to define so-called semi-discrete numerical solver, as briefly discussed in Section~\ref{sec-pw}.  

$\c$-transforms can be generalized to the regularized case $\epsilon \geq 0$ by defining
\begin{align}
	\label{eq-disc-c-transfo}
	\foralls \gD \in \RR^m, \;
	\foralls i \in \range{n}, \quad
	\gD^{\bar \c,\epsilon}_i &\eqdef \epsilon \log(\a_i) + {\min}_\epsilon( \C_{i,\cdot} - \gD ), \\
	\foralls \fD \in \RR^n, \;
	\foralls j \in \range{m}, \quad
	\fD^{\,\c,\epsilon}_j &\eqdef \epsilon \log(\b_j) + {\min}_\epsilon( \C_{\cdot,j} - \fD ), 
\end{align}
where we denoted $\min_0=\min$ the usual minimum of a vector and for $\epsilon>0$, the soft-minimum is defined as
\eq{
	 {\min}_\epsilon \mathbf{u} \eqdef - \epsilon \log \sum_i e^{-\mathbf{u}_i/\epsilon}.
}
One can check that for a fixed $\fD$ (resp. a fixed $\gD$) $\gD = \fD^{\,\c,\epsilon}$ (resp. $\fD=\gD^{\bar \c,\epsilon}$) minimizes~\eqref{eq-dual-formulation} with respect to $\gD$ (resp. with respect to $\fD$). 

The widely used Sinkhorn algorithm to solve the regularized OT problem can be simply interpreted as a block-coordinate ascent method, where, starting from some initialization $\fD^{(0)} \in \RR^n$, one defines, for $\ell \geq 0$
\eq{
	\gD^{(\ell)} \eqdef (\fD^{(\ell)})^{\c,\epsilon}
	\qandq
	\fD^{(\ell+1)} \eqdef (\gD^{(\ell)})^{\bar\c,\epsilon}.
}
For $\epsilon=0$, alternating these $\c$-transforms does not converge because the dual problem is not smooth, and one immediately reaches a stationary point $(\fD^{(\ell+1)},\gD^{(\ell+1)})=(\fD^{(\ell)},\gD^{(\ell)})$ for $\ell=2$.
In sharp contrast, for $\epsilon>0$, these iterations are known to convergence at linear speed, with a dependence given by $\epsilon$ and the values of $\c$, see~\cite{franklin1989scaling}.

%%%%
\paragraph{Semi-dual}

Explicitly minimizing with respect to one of the two variables appearing in~\eqref{eq-dual-formulation} defines a new dual optimization problem, that we coin ``semi-dual'' because it underlies all semi-discrete OT methods in the case $\epsilon=0$. 

\begin{prop}
One has for $\epsilon \geq 0$
\eql{\label{eq-semi-dual-formulation}
	\MKD_\C^\epsilon(\a,\b) = \umax{\fD \in \RR^n}
		 \dotp{\fD}{\a} + \dotp{\fD^{\,\c,\epsilon}}{\b} - \epsilon 
		  = \umax{\gD \in \RR^m}
		 \dotp{\gD^{\bar \c,\epsilon}}{\a} + \dotp{\gD}{\b} - \epsilon.
} 
\end{prop}

In the following, we denote
\eq{
	\FD_\b(\a) \eqdef \FD_\a(\b) \eqdef \MKD_\C^\epsilon(\a,\b).
}
Formula~\eqref{eq-semi-dual-formulation} equivalently reads as an explicit formula for the Fenchel--Legendre transform:
\eq{
	\foralls \fD \in \RR^n, \quad
	\FD_\b^*(\fD) \eqdef \umax{\a \in \Si_n} \dotp{\fD}{\a} - \FD_\b(\a), 
}
as detailed in the following Corollary.

\begin{cor}\label{cor-legendre-single}
One has for $\epsilon \geq 0$, 
\eql{\label{eq-legendre}
	\FD_\b^*(\fD) = - \dotp{\fD^{\,\c,\epsilon}}{\b} + \epsilon 
	\qandq
	\FD_\a^*(\gD) = - \dotp{\gD^{\bar \c,\epsilon}}{\a} + \epsilon .
}
For $\epsilon>0$, $\FD_\b^*$ is $C^\infty$. Its gradient function $\nabla \FD_\b^*(\cdot)$ is $1/\epsilon$ Lipschitz. Its value, gradient, and Hessian at $\fD \in \RR^n$ are, writing $\uD \eqdef e^{\fD/\epsilon}$, $\K \eqdef e^{-\C/\epsilon}$, $\vD \eqdef \frac{\b}{\K^\top  \uD}$ and $\P \eqdef \diag(\uD) \K \diag(\vD)$,
	\begin{equation}\label{eq-obj-dual}
		\begin{aligned}
		\FD_\b^*(\fD) &= \epsilon \left( \HD(\b)+\dotp{\b}{\log \K^\top \uD}\right),\, \nabla \FD_\b^*(\fD) = \uD \circ (\K  \vD) \in \Sigma_n,\\
		\nabla^2 \FD_\b^*(\fD)  &= \frac{1}{\epsilon}\diag(\uD \circ \K \vD) - \frac{1}{\epsilon}\P \diag(\b)^{-1} \P^\top.
		\end{aligned}
	\end{equation}
%	\todo{Gab: maybe need to write the sub-gradient for $\epsilon=0$ since it is referred to in the algorithmic section? Or simply refer to the paper of Guillaume where this is explained. }
\end{cor}

\begin{proof}
	We refer to~\cite{2016-Cuturi-siims} for a proof of these formula. 
\end{proof}

%%%%%%%%%%%%%%%%%%%%%%%%%%%%%%%%%%%%%%%%%%%%%%%%%%%%%%%%%%%%%%%%%%%%%%%%%%%%%%%
\subsection{Generalization to Arbitrary Measures}
\label{sec-gen-measure}

An important feature of OT is that it works in a seamless way for discrete measure and continuous density of mass, using convex optimization over the space of Radon measures. We consider two metric spaces $(\Xx,\Yy)$, and Radon measures on these spaces, as the dual space of continuous functions $(\Cc(\X),\Cc(\Y))$. We denote $(\al,\be) \in \Mm_+^1(\X) \times \Mm_+^1(\Y)$ the set of positive probability measures, so that $\al(\X)=\be(\Y)=1$. 
The discrete setting of the previous section is recovered by restricting measures $\al,\be$ to be weighted sums of Dirac masses
\eq{
	\al=\sum_i \a_i \de_{x_i}
	\qandq
	\be=\sum_j \b_j \de_{y_j}. 
} 

%%%%
\paragraph{Regularization using relative entropy}

One can consider arbitrary measures by replacing the discrete entropy by the relative entropy with respect to the product measure $\d(\al\otimes\be)(x,y) \eqdef \d\al(x)\d\be(y)$, and propose a regularized counterpart to~\eqref{eq-regularized-discr} using
\eql{\label{eq-entropic-generic}
	\MK_\c^\epsilon(\al,\be) \eqdef 
	\umin{\pi \in \Couplings(\al,\be)}
		\int_{\X \times \Y} c(x,y) \d\pi(x,y) + \epsilon \KL(\pi|\al\otimes\be)
}
where the relative entropy is a generalization of the discrete Kullback-Leibler divergence
\eql{\label{eq-defn-rel-entropy}
	\KL(\pi|\xi) \eqdef \int_{\X \times \Y} \log\Big( \frac{\d \pi}{\d\xi}(x,y) \Big) \d\pi(x,y)
	  +\int_{\X \times \Y} (\d\xi(x,y)-\d\pi(x,y)), 
}
and by convention $\KL(\pi|\xi)=+\infty$ if $\pi$ does not have a density $\frac{\d \pi}{\d\xi}$ with respect to $\xi$. Formulation~\eqref{eq-entropic-generic} was initially proposed in~\cite{genevay2016stochastic}.

For generic and not necessarily discrete input measures $(\al,\be)$, the dual problem~\eqref{eq-dual-formulation} reads
% \MC{j'ai chang\'e le sens de $\f \oplus \g -\c$ pour être pareil que le cas discret}
\begin{align*}\label{eq-dual-entropic}
	\MK_\c^\epsilon(\al,\be)  &= \usup{(\f,\g) \in \Cc(\X)\times\Cc(\Y)} \int_\X \f \d\al + \int_\Y \g \d\be 
		 + B_\epsilon(\c- \f \oplus \g ), \\
	&= \usup{(\f,\g) \in \Cc(\X)\times\Cc(\Y)} \int_\X \f \d\al + \int_\Y \g \d\be 
		+ {\min}_\epsilon(\c - \f \oplus \g ), 
\end{align*}
where $B_0 \eqdef - \iota_{\Cc}$ where $\Cc = \enscond{S \in \Cc(\X \times \Y)}{\forall (x,y), S(x,y) \geq 0}$, 
and for $\epsilon>0$
\eq{
	 B_\epsilon(S) = - \epsilon \int_{\X\times\Y} e^{ \frac{-S(x,y)}{\epsilon} } \d\al(x)\d\be(y)
}
%This corresponds to a smoothing of the constraint $\Potentials(\c)$ appearing in the original problem~\eqref{eq-dual-generic}, which is retrieved in the limit $\epsilon \rightarrow 0$. ---> pas besoin?? vu que \Potentials n'est pas introduit.

The definition~\eqref{eq-disc-c-transfo} of the discrete $\c$-transforms now becomes in the general setting
\begin{align}\label{eq-c-transform}
	\foralls y \in \Y, \quad
	\f^{\c,\epsilon}(y) &\eqdef {\inf}_{\epsilon,\al}  \c(\cdot,y) - \f, \\ 
	\foralls x \in \X, \quad
	\g^{\bar\c,\epsilon}(x) &\eqdef {\inf}_{\epsilon,\be} \c(x,\cdot) - \g, 
\end{align}
where we denoted $\bar\c(y,x) \eqdef c(x,y)$. For $\epsilon=0$, ${\inf}_{0,\al}=\inf$ is the usual infimum, while for $\epsilon>0$
\eq{
	\foralls h \in \Cc(\X), \quad {\inf}_{\epsilon,\al} h = -\epsilon \log \int_\X e^{h(x)/\epsilon} \d\al(x) 
}
and similarly for ${\inf}_{\epsilon,\be}$.
Figure~\ref{fig-c-transform-discrete-eps} displays the influence of $\epsilon$ on these $\c$-transforms. 
The semi-dual formulation~\eqref{eq-semi-dual-formulation} then reads for general measures
\eql{\label{eq-semi-dual-formulation-gen}
\begin{aligned}
	\MK_\c^\epsilon(\al,\be) &= 
		\umax{\f \in \Cc(\X)} \int_\X \f(x) \d\al(x)  + \int_\Y \f^{\c,\epsilon}(y) \d\be(y)  \\
		&= \umax{\g \in \Cc(\Y)} \int_\X \g^{\bar \c,\epsilon}(x) \d\al(x) + \int_\Y \g(y) \d\be(y) .
\end{aligned}		 
}

% G/B/D/H \fbox
\newcommand{\MyFigCTransEps}[1]{\includegraphics[width=.24\linewidth,trim=63 40 48 30,clip]{c-transform/c-transform-eps#1}}
\begin{figure}
\centering
\begin{tabular}{@{}c@{\hspace{1mm}}c@{\hspace{1mm}}c@{\hspace{1mm}}c@{}}
\MyFigCTransEps{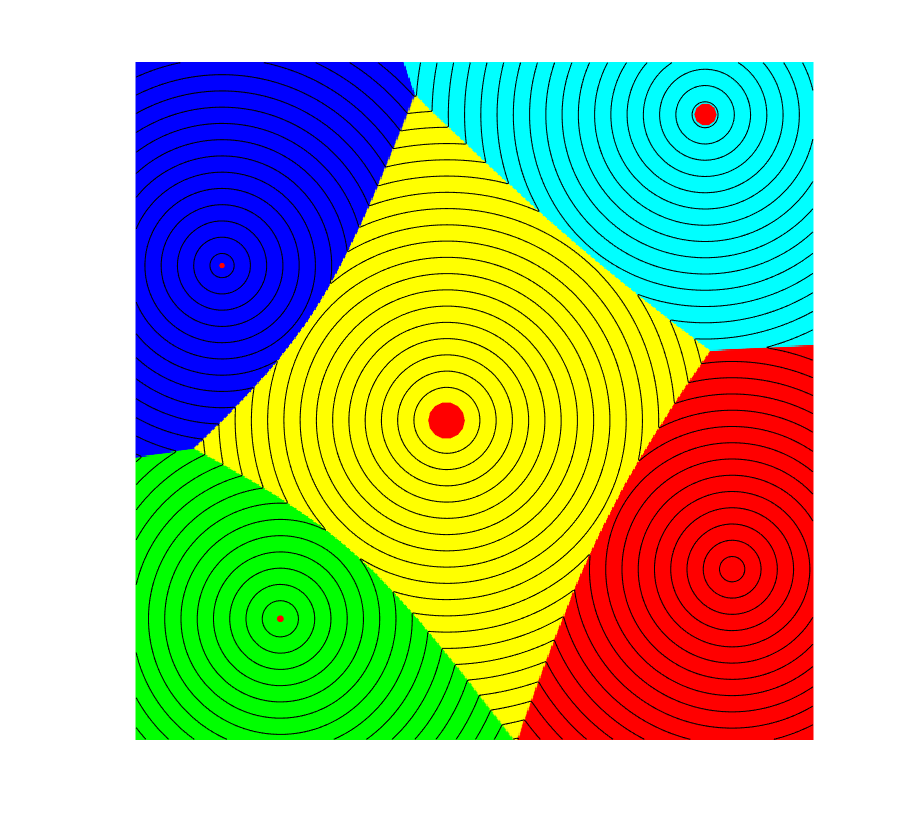}&
\MyFigCTransEps{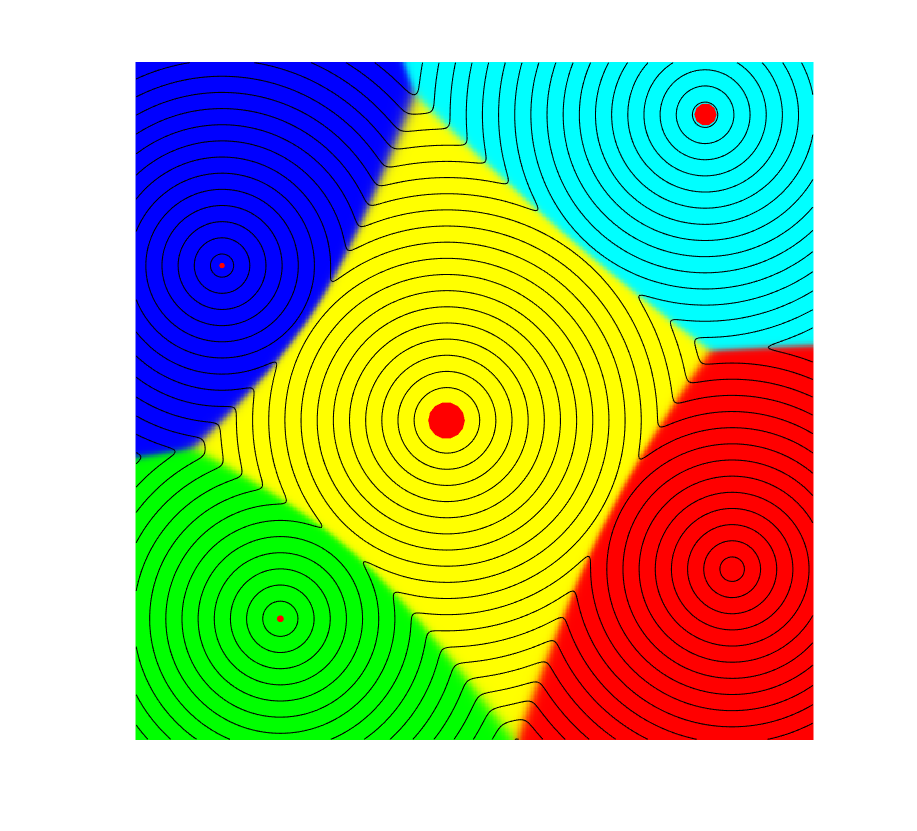}&
\MyFigCTransEps{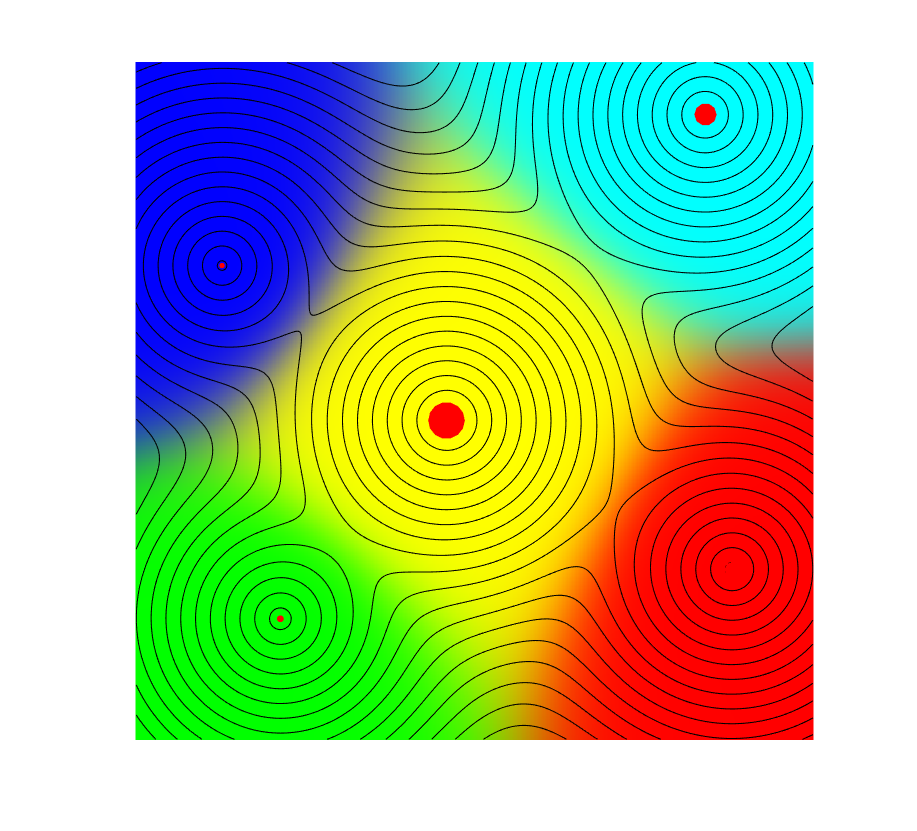}&
\MyFigCTransEps{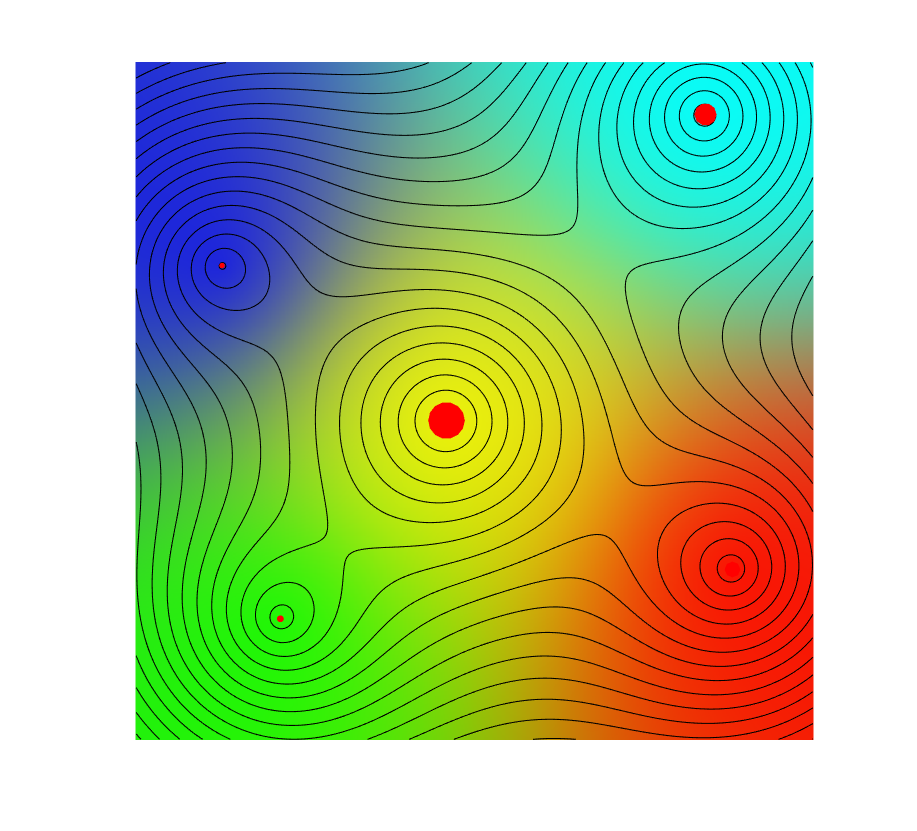}\\
$\epsilon=0$ & $\epsilon=0.01$ & $\epsilon=0.1$ & $\epsilon=0.3$ 
\end{tabular}
\caption{\label{fig-c-transform-discrete-eps}
Examples of entropic semi-discrete $\bar c$-transforms $\gD^{\bar \c,\epsilon}$ in 2-D, for ground cost $c(x,y)=\norm{x-y}$ for varying $\epsilon$. The black curves are the level sets of the function $\gD^{\bar \c,\epsilon}$, while the colors indicate the smoothed indicator function of the Laguerre cells $\chi_j^\epsilon$.
The red points are at locations $y_j \in \RR^2$, and their size is proportional to $\gD_j$. 
}
\end{figure}

%%%%
\paragraph{Semi-discrete solvers}

The semi-discrete case, where $\al$ has a density (with respect to Lebesgue measure) while $\be=\sum_{j=1}^m \b_j \de_{y_j}$ is discrete, is interesting from a computational point of view. In this setting, one can solve the problem through the lens of a variable $\gD = ( g(y_j) )_{j=1}^m \in \RR^m$, by imposing $f=\g^{\bar\c,\epsilon}$ where 
\eq{
	\gD^{\bar\c,\epsilon}(x) \eqdef 
	\choice{
		-\epsilon \log \sum_{j} e^{ -\frac{\c(x,y_j) - \gD_j}{\epsilon} } \qforq \epsilon>0, \\
		\min_{j} \c(x,y_j) - \gD_j \qifq \epsilon = 0.
	}
}
The semi-dual problem then read in this semi-discrete case 
\eql{\label{eq-semi-dual-discr}
	\MK_\c^\epsilon(\al,\be) = 
		\umax{\gD \in \RR^m}
			% \Ee(\gD) \eqdef 
			\Ee^\epsilon(\gD) \eqdef \int_\X \gD^{\bar \c,\epsilon}(x) \d\al(x) + \sum_j \gD_j \b_j.
}
This optimization problem is attractive because the objective function is written as an integral against measure $\al$, so it can be solved using stochastic gradient descent as detailed in~\cite{genevay2016stochastic}.
The gradient of this functional, for any $\epsilon \geq 0$ reads
\eql{\label{eq-grad-semid-entrop}
	\foralls j \in \range{m}, \quad
	\nabla\Ee^\epsilon(\gD)_j = - \int_{\X} \chi_j^\epsilon(x) \d\al(x) + \b_j, 
}
where $\chi_j^\epsilon$ is a smoothed version of the indicator $\chi_j^0$ of the Laguerre cell $\Laguerre_{j}(\gD)$
\eql{\label{eq-smoothed-indic}
	\chi_j^\epsilon(x) = 
	\frac{
		e^{\frac{-\c(x,y_j) + \gD_j}{\epsilon}}
	}{
		\sum_\ell e^{\frac{-\c(x,y_\ell) + \gD_\ell}{\epsilon}}
	}.
}

In the case $\epsilon=0$, one can introduce the Laguerre cells
\eq{
	\Laguerre_{j}(\gD) \eqdef \enscond{x \in \X}{ \foralls j' \neq j, \c(x,y_j) - \gD_j \leq \c(x,y_{j'}) - \gD_{j'} }
}
which induce a disjoint decomposition of $\X = \bigcup_j \Laguerre_{j}(\gD)$. 
When $\gD$ is constant, the Laguerre cells decomposition corresponds to the Voronoi diagram partition of the space. 
Note that $\chi_j^0$ defined in~\eqref{eq-smoothed-indic} (i.e. for $\epsilon=0$) is the indicator function of the Laguerre cell $\Laguerre_{j}(\gD)$. 
A simple algorithm to solve~\eqref{eq-semi-dual-discr} is a gradient descent 
\eql{\label{eq-grad-desc}
	\IIT{\gD} \eqdef (1-\tau_\ell) \IT{\gD} + \tau_\ell \int_{\Laguerre_{j}(\gD)} \d\al(x)
}
for some step size $\tau_\ell>0$, which only requires the computation of the Laguerre cells. 
% More advanced algorithm making use of second order information on the function (such as quasi-Newton) can be used, see~\cite{Merigot11}.
%
In the special case $\c(x,y)=\norm{x-y}^2$, the decomposition in Laguerre cells is also known as a ``power diagram''. 
The cells are polyhedral and can be computed efficiently using computational geometry algorithms, see~\cite{aurenhammer1987power}. 
The most widely used algorithm relies on the fact that the power diagram of points in $\RR^\dim$ is equal to the projection on $\RR^\dim$ of the convex hull of the set of points $( (y_j,\norm{y_j}^2 - \gD_j) )_{j=1}^m \subset \RR^{\dim+1}$. There are numerous algorithms to compute convex hulls, for instance that of~\cite{chan1996optimal} in 2-D and 3-D has complexity $O(m \log(Q))$ where $Q$ is the number of vertices of the convex hull.

\newcommand{\MyFigSemiD}[1]{\includegraphics[width=.19\linewidth]{semidiscrete-gd/#1}}
\begin{figure}
\centering
\begin{tabular}{@{}c@{\hspace{1mm}}c@{\hspace{1mm}}c@{\hspace{1mm}}c@{\hspace{1mm}}c@{}}
\MyFigSemiD{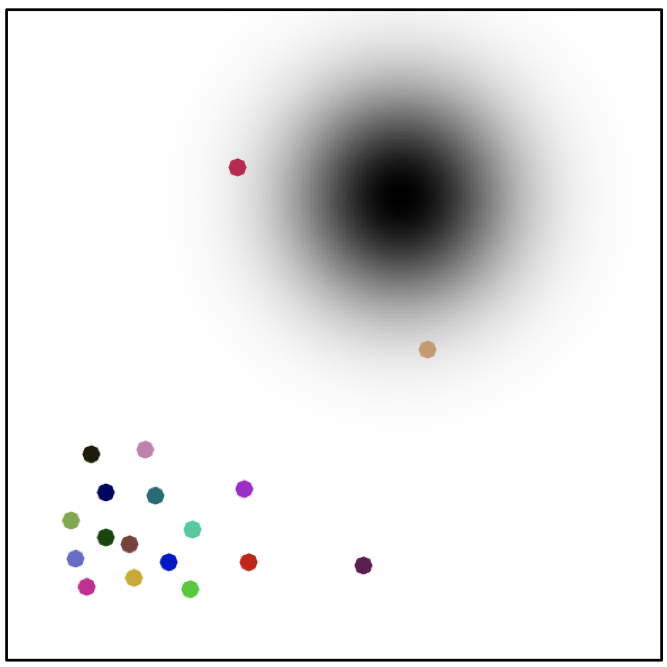} &
\MyFigSemiD{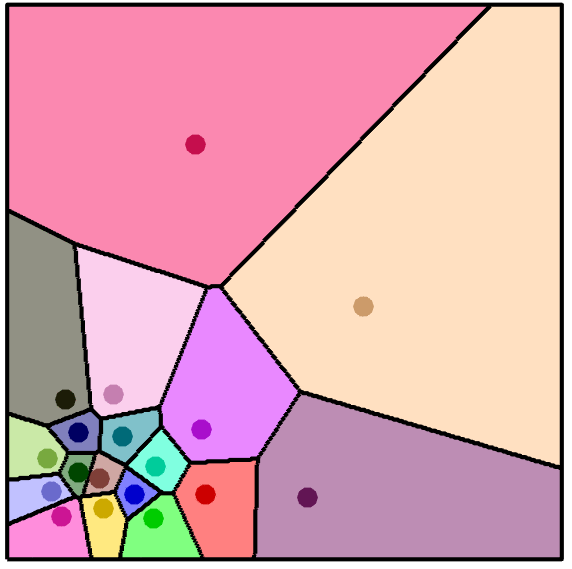} &
\MyFigSemiD{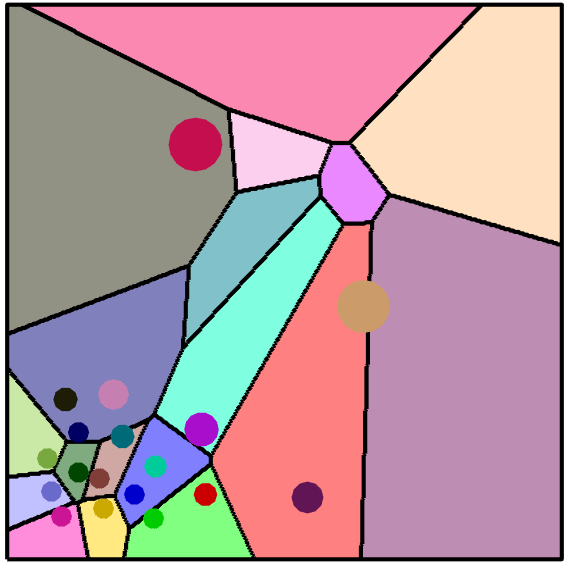} &
\MyFigSemiD{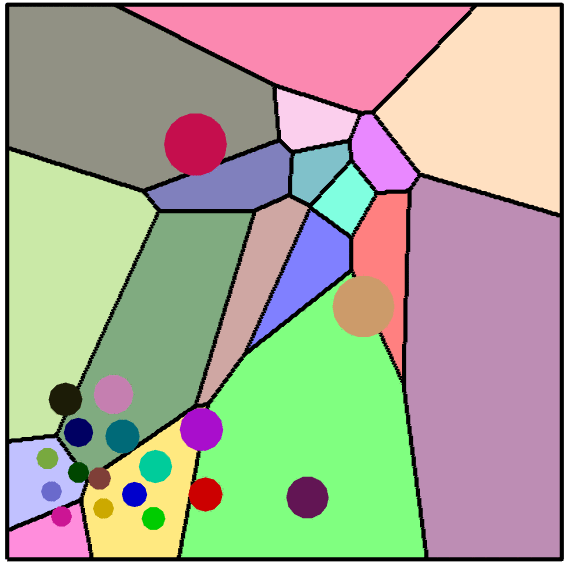} &
\MyFigSemiD{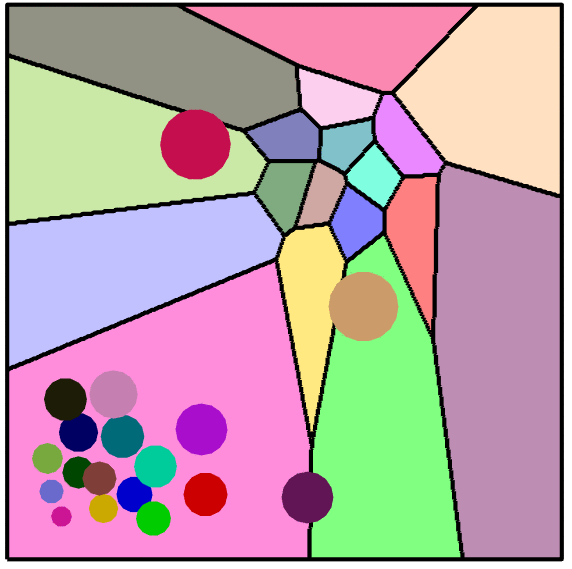}   \\ 
$\al$ and $\be$ &
$\ell=1$ &
$\ell=20$ &
$\ell=40$ &
$\ell=100$ 
\end{tabular}
\caption{\label{fig-semi-discr}
Iterations of the gradient descent~\eqref{eq-grad-desc} to solve the semi-discrete OT (when $\epsilon=0$).
The support $(y_j)_j$ of the discrete measure $\be$ is indicated by the colored points, while the continuous measure $\al$ is the uniform measure on a square. 
The colored cells display the Laguerre partition $( \Laguerre_{j}( \IT{\gD} ) )_j$ where $\IT{\gD}$ is the discrete dual potential computed at iteration $\ell$. 
}
\end{figure}

% \todo{Gab: re-written paragraph below, needs polishing. }

The entropic-regularized semi-discrete formulation~\eqref{eq-semi-dual-discr}, for $\epsilon>0$, corresponds to the minimization of a smooth functional. Indeed, one can generalize Corollary~\ref{cor-legendre-both} to arbitrary measures (see~\cite{genevay2016stochastic}) and in particular the Hessian is upper-bounded by $1/\epsilon$, paving the way for the use of off-the-shelf first and second order optimization methods. In practice, similarly to~\cite{Merigot11} in the un-regularized case, we advocate for the use of a quasi-Newton (L-BFGS) solver. We refer to Section~\ref{sec-algo} for a more detailed discussion in the case of barycenters.
Note that other methods exploiting second order schemes with some regularization have been studied by~\cite{knight2013fast,sugiyama2017tensor,cohen2017matrix,allen2017much}.

\section{Smooth dual algorithms for the WBP}\label{sec:dualalgo}

We now use the semi-dual formulation the Wasserstein distance detailed in Section~\ref{sec-duals-semi-duals} to solve the Wasserstein Barycenter Problem (WBP).
In this section, histograms correspond to measures defined on the same space $\X=\Y$, and thus in the discrete setting, we impose $n=m$. The usual way to proceed is to take $\C_{i,j}=d(x_i,x_j)^\rho$ where $\rho \geq 1$ is some exponent, and $d$ is a distance on the underlying space. Then $(\MKD_\C^0)^{1/\rho}$ defines the so-called (discrete) Wasserstein distance on the simplex $\Si_n$, and $(\MK_\c^0)^{1/\rho}$ metrizes the weak-* convergence of measures (which corresponds to the convergence in law of random variables).

%%%%%%%%%%%%%%%%%%%%%%%%%%%%%%%%%%%%%%%%%%%%%%%%%%%%%%%%%%%%%
%SEC 3.1
\subsection{Smooth dual formulation of the WBP}

Following the introduction of the WBP by~\cite{Carlier_wasserstein_barycenter},~\cite{cuturi2014fast} introduced the smoothed WBP with $\epsilon$-entropic regularization ($\epsilon$-WBP) as
\eql{\label{eq-variational-barycenter-discrete}
	\umin{\a \in \Si_n}
		\sum_{k=1}^N \bweight_k \MKD_\C^\epsilon(\a,\b_k) = \sum_{k=1}^N \bweight_k  \FD_{\b_k}(\a),
}
where $(\b_1,\ldots ,\b_N)$ is a family of histograms in $\Si_n$ and $\bweight \in \Si_N$ is a set of weights. When $\epsilon=0$, the $\epsilon$-WBP is exactly the WBP. In that case, problem \eqref{eq-variational-barycenter-discrete} is in fact a linear program, as discussed later in Section~\ref{sec:wassbarprob}. When $\epsilon>0$ the $\epsilon$-WBP is a \emph{strictly} convex optimization problem that admits a unique solution, which can be solved with a simple gradient descent as advocated by \cite{cuturi2014fast}. They show that the $N$ gradients $\left[\nabla \FD_{\b_k}(\a) \right]_{k \leq \N}$ can be computed at each iteration by solving $\N$ Sinkhorn matrix-scaling problems. Because these gradients are themselves the result of a numerical optimization procedure, the problem of choosing a thresholds to obtain sufficient accuracy on these gradients arises.
We take here a different route to solve the $\epsilon$-WBP, which can be  interpreted either as a smooth alternative to the dual WBP studied by~\cite{Carlier-NumericsBarycenters} or the dual counterpart to the smoothed WBP of \cite{cuturi2014fast}.

%thm 3.1
\begin{theorem}\label{prop-dual-energy}
	The barycenter $\a^\star$ solving~\eqref{eq-variational-barycenter-discrete} satisfies 
	\eql{\label{eq-primal-dual-relationship}
		\foralls k = 1,\ldots,\N, \quad 
		\a^\star = \nabla \FD_{\b_k}^*(\fD_k^\star),
	}
	where $( \fD_k^\star )_k$ are any solution of the smoothed dual WBP:
	\eql{\label{eq-dual-pbm}
		\umin{ \fD_1,\ldots, \fD_N\in\RR^n} \sum_k \bweight_k \FD_{\b_k}^*(\fD_k)
		\qstq \sum_k \bweight_k \fD_k = 0.
	}
\end{theorem}
\unskip

\begin{proof}
	We rewrite the barycenter problem 
	$$\umin{\a_1,\ldots,\a_N} \sum_k \bweight_k \FD_{\b_k}(\a_k) \qstq \a_1=\cdots=\a_N$$
	whose Fenchel--Rockafelar dual  reads
		$$\umin{\tilde \fD_1,\ldots,\tilde \fD_N} \sum_k \bweight_k \FD_{\b_k}^*(\tilde \fD_k/\bweight_k) 
		\qstq \sum_k \tilde \fD_k = 0.$$
	Since the primal problem is strictly convex, the primal-dual relationships show that the unique solution $\a^\star$ of the primal can be obtained from any solution $(\tilde\fD_k^\star)_k$ via the relation
	$\a_k^\star = \nabla \FD_{\b_k^\star}^*(\tilde \fD_k^\star/\bweight_k)$.
	One obtains the desired formulation using the change of variable $\fD_k = \tilde \fD_k/\bweight_k$.% \qquad
\end{proof}

Theorem~\ref{prop-dual-energy} provides a simple approach to solve the $\epsilon$-WBP: Rather than minimizing directly the sum of regularized Wasserstein distances in \eqref{eq-variational-barycenter-discrete}, this formulation only involves minimizing a strictly convex function with closed form objectives and gradients.

%%%%%
\paragraph{Parallel implementation} 

%\todo{Needs revisiting with new notation most likely.}

The objectives, gradients, and Hessians of the Fenchel--Legendre dual $\FD^*_\b$ can be computed using either matrix-vector products or elementwise operations. Given $N$ histograms $(\b_k)_k$ and $N$ dual variables $(\fD_k)_k$, the computation of $N$ objective values $(\FD_{\b_k}^*(\fD_k))_k$ and $N$ gradients $(\nabla \FD_{\b_k}^*(\fD_k))_k$ can all be vectorized. Assuming that all column vectors $\fD_k$ and $\b_k$ are gathered in $n\times N$ matrices $F$ and $B$ respectively, we first define the  $n\times N$ auxiliary matrices
\eq{
	A\defeq e^{F/\epsilon}, \quad C\defeq \frac{B}{\K A} , \quad  \Delta\defeq A\circ (\K C)
}
to form the vector of objectives
\eql{\label{eq:onehistoobj}
	H^*\; \defeq [ \FD_{\b_1}^*(\fD_1),\dots, \FD_{\b_N}^*(\fD_N)]= -\epsilon \ones_n^\top \left(B \circ \log(C)\right)
}
and the matrix of gradients 
\begin{equation}\label{eq:onehistograd}\nabla H^*\; \defeq [\nabla \FD_{\b_1}^*(\fD_1),\dots,\nabla \FD_{\b_N}^*(\fD_N)]=\Delta.\end{equation}

%%%%%%%%%%%%%%%%%%%%%%%%%%%%%%%%%%%%%%%%%%

% 
% \subsection{Parallel Computations, Two Histograms} 
% 
% We precompute the denominators appearing in Eq.~\eqref{eq:xgh} for all pairs $(\fD_k,\FD_k)_{k\leq N}$ using a $n\times n$ matrix times $n\times N$ matrix of variables, a Schur product and a simple matrix vector product: 
% $$U=\ones_n^\top \left(e^{G/\epsilon}\circ (Ke^{H/\epsilon})\right).$$
% $U$ is therefore a $1 \times N$ row vector.
% The $2n\times N$ matrix of gradients is then obtained as
% $$\nabla W^*\;\defeq [\nabla W^*(\fD_1,\FD_1),\dots,\nabla W^*(\fD_N,\FD_N)]= \begin{bmatrix}e^{G/\epsilon} \circ (\K e^{H/\epsilon}) \\ e^{H/\epsilon} \circ (\\K  e^{G/\epsilon})\end{bmatrix}\circ \left(\ones_{2n} 1/U \right),$$
% where the product by a diagonal matrix above should be implemented in practice using a fast subroutine such as \texttt{bsxfun} in MATLAB, broadcasting with Numpy in Python or any convenient way to scale rows in a matrix.

%%%%%%%%%%%%%%%%%%%%%%%%%%%%%%%%%%%%%%%%%%%%%%%%%%%%%%%%%%%%%
%SEC 3.2
\subsection{Algorithm}
\label{sec-algo}

The $\epsilon$-WBP in \eqref{eq-dual-pbm} has a smooth objective with respect to each of its variables $\fD_k$, a simple linear equality constraint, and both gradients and Hessians that can be computed in closed form. We can thus compute a minimizer for that problem using a naive gradient descent outlined in Algorithm~\ref{algo:firstorder}.
Note that the iterates $F$ are projected at each iteration on the constraint $F \bweight=0$ (which is equivalent to projecting the gradient direction on this constraint if the initial $F$ satisfies it).
To obtain a faster convergence, it is also possible to use accelerated gradient descent, quasi-Newton or truncated Newton methods~\cite[section~10]{Boyd:1072}. In the latter case, the resulting KKT linear system is sparse, and solving it with preconjugate gradient techniques can be efficiently carried out. We omit these details and only report results using off-the-shelf L-BFGS. From the dual iterates $\fD_k$ stored in a $n\times N$ matrix $F$, one recovers primal iterates using the formula~\eqref{eq-primal-dual-relationship}, namely, $\a_k = e^{\fD_k/\epsilon}\circ \K \frac{\b_k}{\K e^{\fD_k/\epsilon}}.$
At each intermediary iteration one can thus form a solution to the smoothed Wasserstein barycenter problem by averaging these primal solutions, $\tilde{\a} = \Delta\ones_N/N.$ Upon convergence, these $\a_k$ are all equal to the unique solution $\a^\star$. The average at each iteration $\tilde{\a}$ converges toward that unique solution, and we use the sum of all linewise standard deviations of $\Delta$: $\ones_d^\top\sqrt{(\tilde{\Delta}\circ \tilde{\Delta}) \ones_N/N}$, where $\tilde{\Delta}=\Delta(I_N-\frac{1}{N}\ones_N\ones_N^\top)$ to monitor that convergence in our algorithms.

%alg 1
\begin{algorithm}
	\begin{algorithmic}[1]
		\caption{\hspace*{-3pt}. Smoothed Wasserstein Barycenter, Generic Algorithm.\label{algo:firstorder}}
		\STATE \textbf{Input}: $B=[\b_1,\ldots,\b_N] \in(\Sigma_n)^N$, metric $M\in\RR_+^{n\times n}$, barycenter weights $\bweight\in\Sigma_N$, $\epsilon>0$, tolerance $\varepsilon>0$. 
		\STATE initialize $F\in\mathbb{R}^{n\times N}$ and form the $n\times n$ matrix $K=e^{-\C/\epsilon }$.
		\REPEAT
				\STATE From gradient matrix $\Delta$ (see~(\ref{eq:onehistograd})) produce update matrix $\hat{\Delta}$ using either $\Delta$ directly or other methods such as L-BFGS.
				\STATE $F= F - \tau \hat{\Delta}$, update with fixed step length $\tau$ or approximate line search to set $\tau$.
				\STATE $F = F - \frac{1}{\norm{\bweight}_2^2}(F \bweight)  \bweight^\top$ \quad (projection such that $G\bweight=0$)
		\UNTIL{$\ones_d^\top\sqrt{(\tilde{\Delta}\circ \tilde{\Delta}) \ones_N/N}<\varepsilon$, where $\tilde{\Delta}=\Delta(I_N-\frac{1}{N}\ones_N\ones_N^\top)$}
		\STATE output barycenter $\a=\Delta \ones_N/N$.
	\end{algorithmic}
\end{algorithm}

%\input sections/sec-experiments
%SEC 3.4
\subsection{Smoothing and stabilization of the WBP}
\label{sec:wassbarprob}
We make the claim in this section that smoothing the WBP is not only beneficial computationally but may also yield more stable computations. Of central importance in this discussion is the fact that the WBP can be cast as a LP of $Nn^2+n$ variables and $2Nn$ constraints and thus can be solved \emph{exactly} for small $n$ and $N$:
\begin{equation}\begin{aligned}\label{eq-WBP-simplex}
	\min_{\P_1,\ldots,\P_N,p}& \sum_{k=1}^N \bweight_k  \dotp{\P_k}{\C}\\
	\text{s.t. } & \P_k \in \RR^{n\times n}_+ \ \forall k\leq N; \a \in \Sigma_n,\\
	& \P_k^\top \ones_n = \b_k \ \forall k\leq N,\\
	 & \P_1 \ones_n = \dots = \P_N \ones_n = \a.
\end{aligned}\end{equation}
Given couplings $\P_1^\star,\ldots,\P_N^\star$ which are optimal solutions to \eqref{eq-WBP-simplex}, the solution to the WBP is equal to the marginal common to all those couplings: $\a^\star=\P_k^\star \ones_n$ for any $k\leq N$. For small $N$ and $n$, this problem is tractable, but it can be quite ill-posed as we see next.

Indeed, it is known that the $2$-Wasserstein mean of two univariate (continuous) Gaussian densities of mean and standard deviation $(\mu_1,\sigma_1)$ and $(\mu_2,\sigma_2)$, respectively, is a 
Gaussian of mean $(\mu_1+\mu_2)/2$ and standard deviation $(\sigma_1+\sigma_2)/2$~\cite[section~6.3]{Carlier_wasserstein_barycenter}. This fact is illustrated in the top left plot of Figure~\ref{fig:smoothvsnonsmooth}, where we display the Wasserstein average $\Ncal(0,5/8)$ of the two densities $\Ncal(2,1)$ and $\Ncal(-2,1/4)$. That plot is obtained by using smoothed spline interpolations of a uniformly spaced grid of $100$ values, as can be better observed in the top-right (stair) plot, where the discrete evaluations of these densities are respectively denoted $\a_W$, $\b_1$, and $\b_2$.

% transform : p_W -> a_W, q_1,q_2 -> b_1,b_2,  p^* -> a^*

%fig 1
\begin{figure}%[ht]
%	\centering\includegraphics[width=.75\columnwidth]{smoothvsnonsmooth.pdf}    
	\centerline{\includegraphics[width=.75\linewidth]{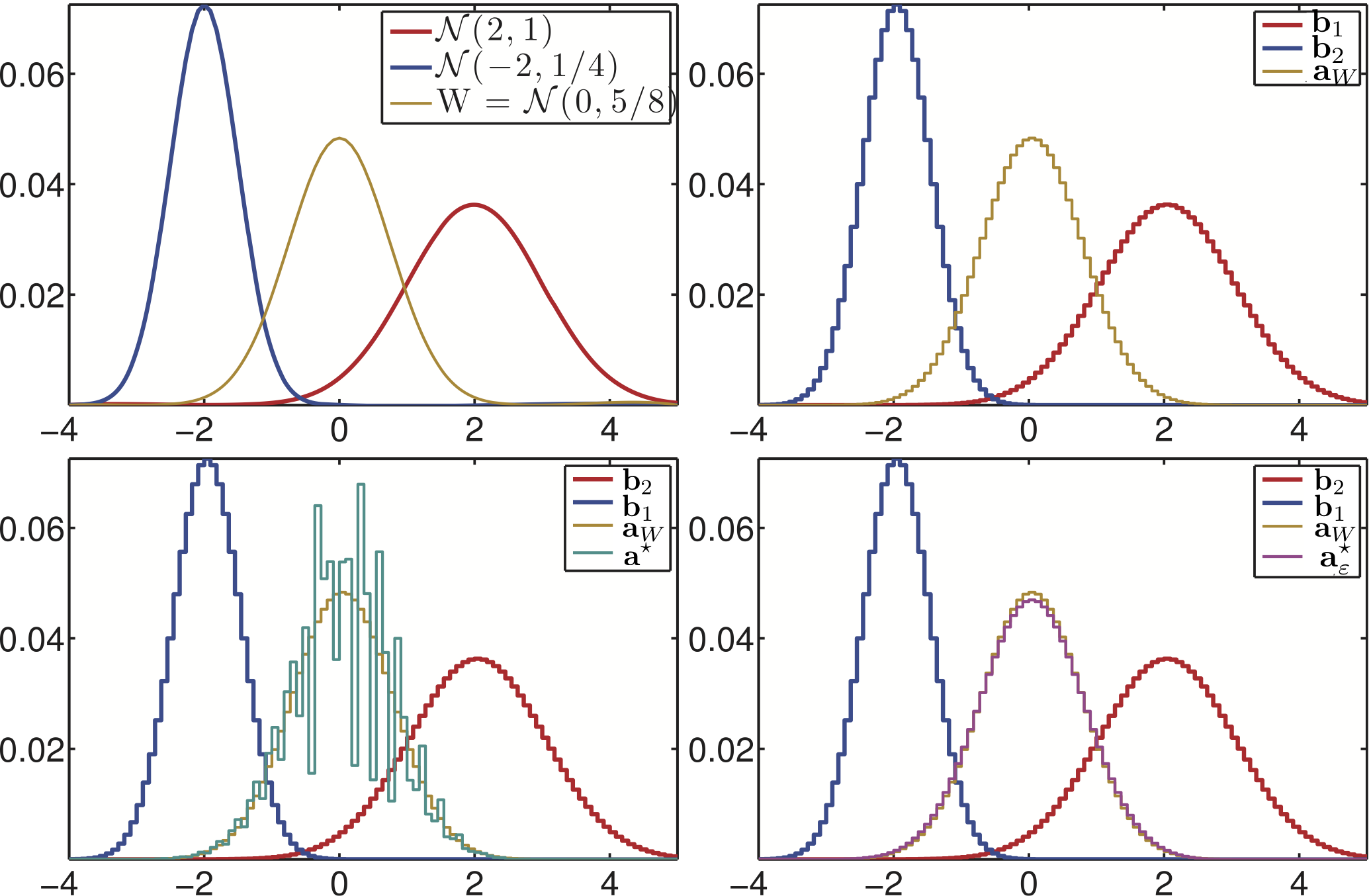}}
	\caption{Top left: two Gaussian densities and their barycenter. Top right: same densities, discretized. Bottom left: 
discretization of the true barycenter vs.\ the optimum of {\rm(\ref{eq-WBP-simplex})}. Bottom right: 
barycenter computed with our smoothing approach.}\label{fig:smoothvsnonsmooth}
\end{figure}

Naturally, one would expect the barycenter of $\b_1$ and $\b_2$ to be close, in some sense, to the discretized histogram $\a_W$ of their true barycenter. Histogram $\a^\star$, displayed in the bottom-left plot, is the exact optimal solution of \eqref{eq-WBP-simplex}, computed with the simplex method. That WBP reduces to a linear program of $2\times 100^2$ variables and $300$ constraints. We observe that $W_2^2(\a^\star,\b_1)+W_2^2(\a^\star,\b_2)=0.5833950$ whereas  $W_2^2(\a_W,\b_1)+W_2^2(\a_W,\b_2)=0.5834070$. The solution obtained with the network simplex has, indeed, a smaller objective than the discretized version of the true barycenter.

The bottom right plot displays the solution of the \emph{smoothed} WBP (with smoothing parameter $\epsilon=\frac{1}{100}$ and a ground cost $\C$ that has been rescaled to have a median value of $1$). The objective value for that smoothed approximation is $0.5834597$.

This numerical experiment does not contradict the fact that the discretized barycenter $\a^\star$ converges (in the weak* sense) to the continuous barycenter as the grid size tends to zero, as shown in~\cite{Carlier-NumericsBarycenters}. This observation illustrates however that, because it is defined as the $\argmin$ of a linear program, the true Wasserstein barycenter may be  unstable (when viewed as an histogram, and not in the sense of the weak* topology of measures), even for such a simple problem and for large $n$ as illustrated in Figure~\ref{fig:smoothvsnonsmooth2}. Regularizing the Wasserstein distances thus has the added benefit of smoothing the resulting solution of the WBP and that of mitigating low sample size effects.

The choice of the parameter $\epsilon$ is application-dependent, but it should scale with the typical distance between sampling locations (e.g., the grid step size). Note that choosing too small $\epsilon$ not only leads to slower convergence of our algorithm but also leads to numerical instabilities and can ultimately break the convergence. In simple settings such as low-dimensional grids, computational strategies can overcome some of these issues~\cite{Schmitzer2016}.

% ------> Version avec une seule figure, mises cÃ´te Ã  cÃ´te. je trouve que les fonts sont trop petites.
%
% \begin{figure}[ht]
% 	\includegraphics[width=.5\linewidth]{smoothvsnonsmooth.pdf}    
% 	\includegraphics[width=.5\linewidth]{nonsmoothsnot_conv.pdf}
% 	\caption{(top-left) two Gaussian densities and their barycenter (top middle) same densities, discretized (bottom left) discretization of the true barycenter \emph{vs.} the optimum of equation~\ref{eq-WBP-simplex} (bottom middle) barycenter computed with our smoothing approach. (right) Plots of the exact barycenters for varying grid size $n$}\label{fig:smoothvsnonsmooth}
% \end{figure}

%fig 2
\begin{figure}%[ht]
%\centering\includegraphics[width=.75\columnwidth]{nonsmoothsnot_conv.pdf}    
\centerline{\includegraphics[width=.75\linewidth]{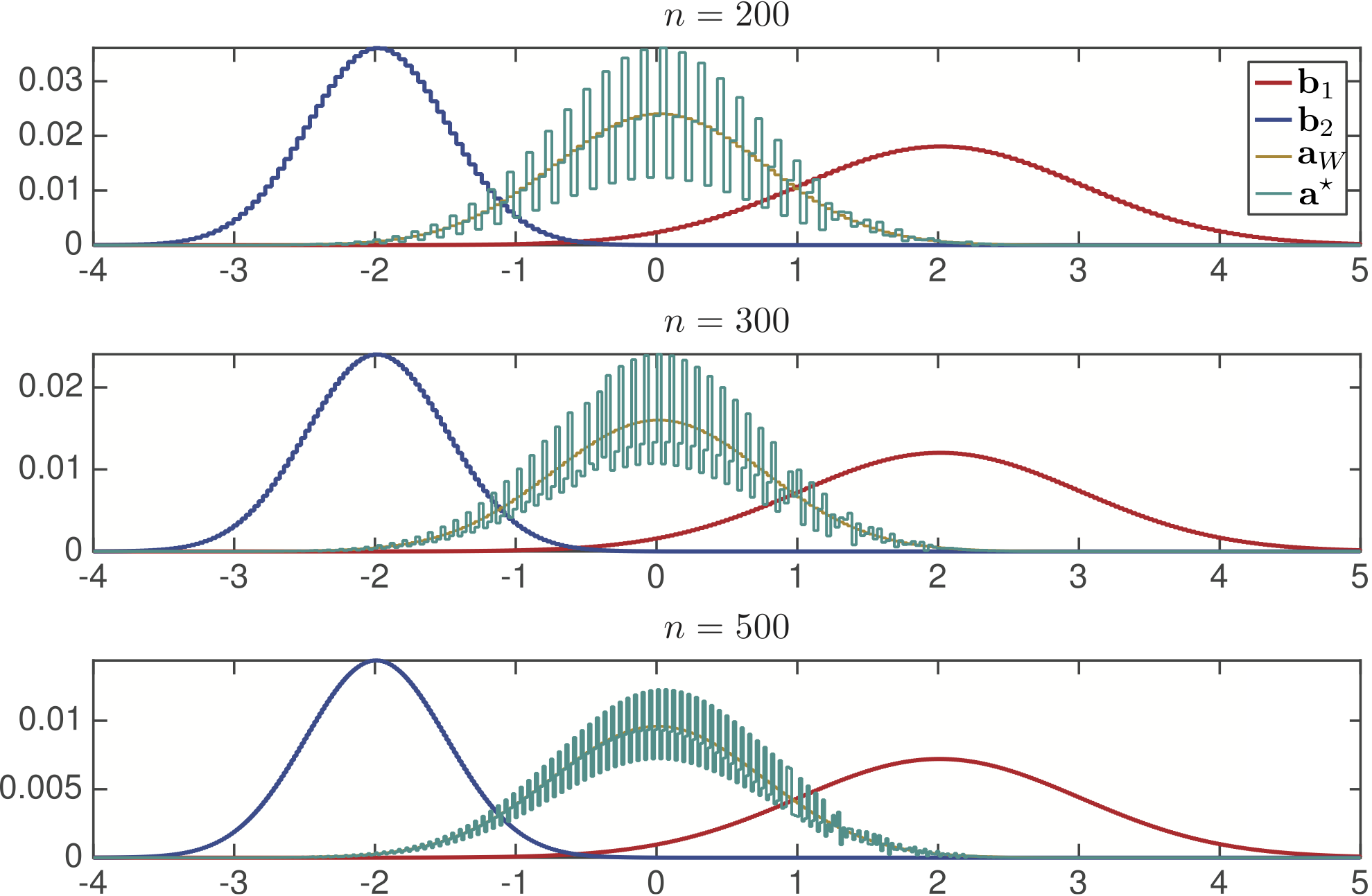}}
\caption{Plots of the exact barycenters for varying grid size $n$.}
\label{fig:smoothvsnonsmooth2}
\end{figure}

%%%%%%%%%%%%%%%%%%%%%%%%%%%%%%%%%%%%%%%%%%%%
%SEC 3.5
\subsection{Performance on the WBP}

We compare in this section the behavior of the smooth dual approach presented in this paper with that of (i) the smooth primal approach of~\cite{cuturi2014fast}, (ii) the dual approach of \cite{Carlier-NumericsBarycenters}, and (iii) the Bregman iterative projections approach of \cite{BenamouCCNP15}. We compare these methods on the simple task of computing the Wasserstein barycenter of $12$ histograms laid out on the $100\times100$ grid, as previously introduced in \cite[section~3.2]{BenamouCCNP15}. We outline briefly all four methods below and follow by presenting numerical results.

% \begin{figure}[h!]
% \centering\includegraphics[width=\columnwidth]{mymixture.png}
% 	\caption{12 measures, truncated mixtures of Gaussians, used in our benchmark. Convergence speed results displayed in Figure~\ref{fig:costs3} and barycenters obtained in~\ref{fig:bar}}.\label{fig:mix}
% \end{figure}

% \begin{figure*}[ht]
% 	\centering\includegraphics[width=\textwidth]{12barycenters.png}
% 	\caption{Barycenters obtained for the three different techniques using the data described in Figure~\ref{fig:mix} after at most $10^4$ iteration units, each iteration unit being equal to $n^2N$ operations, here $(100\times 100)^2\times 12$.}\label{fig:bar}
% \end{figure*}

\paragraph{Smooth primal first order descent} Cuturi and Doucet \cite[section~5]{cuturi2014fast} proposed to directly minimize \eqref{eq-variational-barycenter-discrete} with a regularizer $\epsilon>0$. That objective can be evaluated by running $N$ Sinkhorn fixed-point iterations in parallel. That objective is differentiable, and its gradient is equal to $\epsilon\sum_k \bweight_k \log \uD_k$, where the $\uD_k$ are the left scalings obtained with that subroutine. A weakness of that approach is that a tolerance $\epsilon$ for the Sinkhorn fixed-point algorithm must be chosen. Convergence for the Sinkhorn algorithm can be measured with a difference in $\ell_1$ norm (or any other norm) between the row and column marginals of $\diag(\uD_k)e^{-\C/\epsilon}\diag(\vD_k)$ and the targeted histograms $\a$ and $\b_k$. Setting that tolerance $\epsilon$ to a large value ensures a faster convergence of the subroutine but would result in noisy gradients which could slow the convergence of the algorithm. Because the smoothed dual approach only relies on closed form expressions, we do not have to take into account such a trade-off.

\paragraph{Iterative Bregman projections} Benamou et al.\ \cite[Proposition~1]{BenamouCCNP15} recall that the computation of the smoothed Wasserstein distance between $\a,\b$ using the Sinkhorn algorithm can be interpreted as an iterative alternated projection of the $n \times n$ kernel matrix $e^{-\C/\epsilon}$ onto two affine sets, $\enscond{\P}{\P\ones_n=\a}$ and $\enscond{\P}{\P^\top\ones_n=\b}$. That projection is understood to be in the Kullack--Leibler divergence sense. More interestingly, the authors also show that the smoothed WBP itself can also be tackled using an iterative alternated projection, cast this time in a space of dimension $ n\times n \times N$. Very much like the original Sinkhorn algorithm, these projections can be computed for a cheap price, by only tracking variables of size $n\times N$. This approach yields an extremely simple, parameter-free generalization of the Sinkhorn algorithm which can be applied to the WBP.

\paragraph{Smooth dual L-BFGS} 
The dual formulation with variables $(\fD_1,\ldots,\fD_N)\in(\RR^n)^N$ of \eqref{eq-dual-pbm} can be solved using a constrained L-BFGS solver
At each iteration of that minimization, we can recover a feasible solution $\a$ to the primal problem of \eqref{eq-variational-barycenter-discrete} via the primal-dual relation
$\a = \frac{1}{N}\sum_k \nabla \FD_{\b_k}^*(\tilde \fD_k)$.

\paragraph{Dual $(\epsilon=0)$ with L-BFGS} This approach amounts to solving directly the (nondifferen\-tiable) dual problem described in \eqref{eq-dual-pbm} with no regularization, namely, $\epsilon=0$. Subgradients for the Fenchel--Legendre transforms $\FD_{\b_k}^*$ can be obtained in closed form as detailed in~\cite{Carlier-NumericsBarycenters}.
% through Proposition~\ref{prop:harddual}.  \todo{Gab: I have removed this}
As with the smoothed-dual formulation, we can also obtain a feasible primal solution by averaging subgradients. We follow the recommendation of \cite{Carlier-NumericsBarycenters} to use L-BFGS. The nonsmoothness of that energy is challenging: we have observed empirically that a naive subgradient method applied to that problem fails to converge in all examples we have considered, whereas the L-BFGS approach converges, albeit without guarantees.

%, each a mixture of truncated Gaussians on the $100\times 100$ grid, considered to compute the barycenters displayed in Figure
\paragraph{Averaging truncated mixtures of Gaussians}
\looseness=1We consider the 12 truncated mixtures of Gaussians introduced in \cite[section~3.2]{BenamouCCNP15}. To compare computational time, we use $Nn^2$ elementary operations as the computation unit. These $Nn^2$ operations correspond to matrix-matrix products in the smoothed Wasserstein case and $Nn$ computations of nearest neighbor assignments among $n$ possible neighbors. Note that in both cases (Gaussian matrix product and nearest neighbors under the $\ell_2$ metric) computations can be accelerated by considering fast Gaussian convolutions and $kd$-trees for fast nearest neighbor search. We do not consider them in this section. We plot 
in Figure \ref{fig3} the optimality gap 
with respect to the optimum of these four techniques as a function of the number of computations, by taking as a reference the lowest value attained across all methods. This value is attained, as in \cite{BenamouCCNP15}, by the iterative Bregman projections approach after 771 iterations (not displayed in our graph). We show these gaps for both the smoothed $(\epsilon=1/100)$ and nonsmoothed objectives $(\epsilon=0)$. We observe that the iterative Bregman approach outperforms all other techniques.
The smoothed-dual approach follows closely, and its performance is significantly improved when using the heuristic initialization technic detailed in~\cite{2016-Cuturi-siims}.
% notably when initialized with the formula provided in Definition~\ref{def:SI}. \todo{Gab: removed}

%fig 3
\begin{figure}%[h]
%	\centering\includegraphics[width=\columnwidth]{experiments.pdf}
%	\centerline{\includegraphics[width=.8\columnwidth]{M103260fig-3}}
	\includegraphics[width=.8\columnwidth]{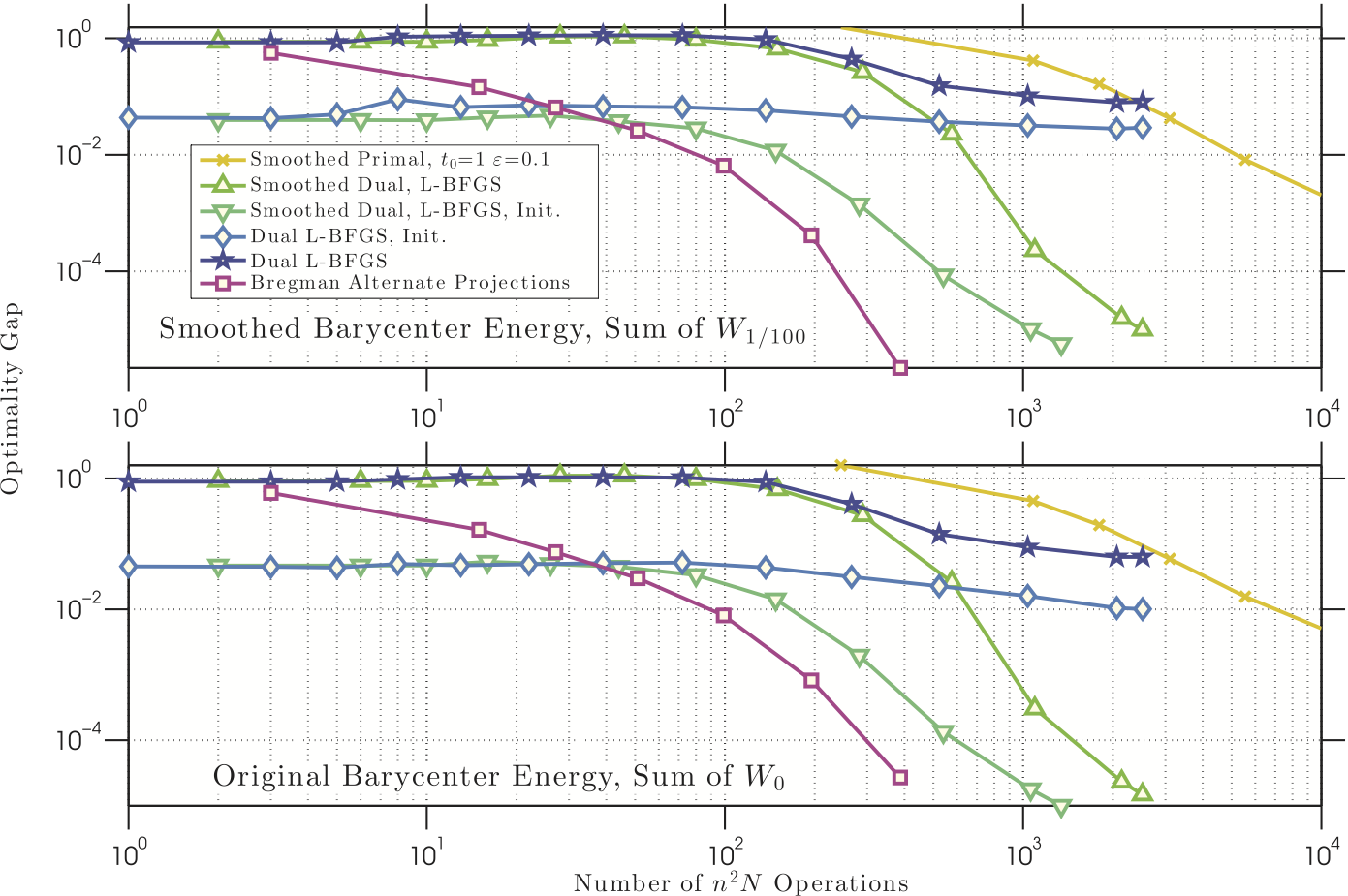}
	\caption{Number of quadratic operations (matrix vector product or min search in a matrix) vs.\ optimization gap to the smallest possible objective found after running $10^5$ iterations of all algorithms, 
\emph{log-log} scale. Because the smooth primal/dual approaches optimize a different criterion than the dual approach, we plot both objectives.  The Smooth dual L-BFGS converges faster in both smooth 
and non-smooth metrics. Note the impact of the initialization detailed in~\cite{2016-Cuturi-siims}.}
\label{fig3}
\end{figure}

% We also plot the solutions of all 3 algorithms in Figure~\ref{fig:bar} after up to $10^4$ iterations. Because the smoothed-primal and our smoothed-dual approach aim at minimizing the same objective, it is not surprising that their solutions are similar. Note however that with a budget of at most $10^4$ iterations the solution obtained with dual smoothing is more detailed than the one obtained with a primal descent. The right-most figure, obtained following \cite{Carlier-NumericsBarycenters}'s approach, shows that the original Wasserstein barycenter problem formulation, without smoothing, can yield very irregular solutions, as was also observed by the authors themselves in their paper.

% !TEX root = ../SIGEST-SemiDual.tex

%SEC 4
%\input sections/sec-regularized-barycenters
\section{Regularized problems}
\label{sec:exten}

We show in this section that our dual optimization framework is versatile enough to deal with functionals involving Wasserstein distances that are more general than the initial WBP problem.

%%%%%%%%%%%%%%%%%%%%%%%%%%%%%%%%%%%%%%%%%%%%%%%%%%%%
%SEC 4.1
\subsection{Regularized Wasserstein barycenters}\label{sec:regularized}

In order to enforce additional properties of the barycenters, it is possible to penalize~\eqref{eq-variational-barycenter-discrete} with an additional convex regularization, and consider%\vspace{-4pt}
\eql{\label{eq-regul-primal}
	\umin{\a} \sum_{k=1}^N \bweight_k \FD_{\b_k}(\a) + J(\Aa \a),
}
where $J$ is a convex real-valued function, and $\Aa$ is a linear operator.

The following proposition shows how to compute such a regularized barycenter through a dual optimization problem. 

%prop 1
\begin{proposition}
The dual problem to~\eqref{eq-regul-primal} reads
\eql{\label{eq-regul-dual}
	\umin{(\fD_k)_{k=1}^N, \gD} \sum_{k=1}^N \bweight_k \FD_{\b_k}^*(\fD_k) + J^*(\gD) + \iota_{H}((\fD_k)_k,\gD),
}
\eq{
	\qwhereq
	H \eqdef \enscond{ ((\fD_k)_{k=1}^N, \gD) }{ \Aa^* \gD + \sum_k \bweight_k \fD_k = 0 },
}
\noindent and the primal-dual relationships read
\eql{\label{eq-primal-dual}
	\foralls k=1,\ldots,N, \quad \a = \nabla \FD_{\b_k}^*(\fD_k). 
}
\end{proposition}
% \unskip

\begin{proof}
	We rewrite the initial program~\eqref{eq-regul-primal} as 
	\eql{\label{eq-primal-reformulated}
		\umin{\pi} F( B\pi ) + G(\pi),
	} 
	where we denoted, for $\pi=(\a,\a_1,\ldots,\a_N)$, 
	\begin{align*}
			B \pi &\eqdef (\Aa \a, \a, \a_1,\ldots,\a_N), \\
			F(\beta,\b, \a_1,\ldots,\a_N) &\eqdef J(\beta) + \iota_{C}(\b,\a_1,\ldots,\a_N),\\
			G(\a, \a_1,\ldots,\a_N) &\eqdef \sum_k \bweight_k \FD_{\b_k}(\a_k)
	\end{align*} 
	for $C \eqdef \enscond{(\b,\a_1,\ldots,\a_N)}{\forall k, \a_k=\b}$.
	The Fenchel--Rockafelear dual to~\eqref{eq-primal-reformulated} reads
	\eq{
		\umax{\nu = \{\gD,\fD,(\fD_k)_k\}} - F^*(\nu) - G^*(-B^* \nu), 
	}
	where 
	\begin{align*}
		G^*(\fD,\fD_1,\ldots,\fD_N) &= \sum_k \bweight_k \FD_{\b_k}^*(\fD_k/\bweight_k) + \iota_{\{0\}}( \fD ),  \\
		B^*(\nu) &= (\Aa^* \gD + \fD, \fD_1,\ldots,\fD_N), \\
		F^*(\nu) &= J^*(\gD) + \iota_{C^\bot}(\fD,\fD_1,\ldots,\fD_N), 
	\end{align*} 
	where $C^\bot = \enscond{ (\fD,\fD_1,\ldots,\fD_N) }{\fD + \sum_k \fD_k=0}$. 
	One thus obtains the dual
	\eq{
		\umin{ \gD,b,(\fD_k)_k } \sum_k \bweight_k \FD_{\b_k}^*(-\fD_k/\bweight_k) + J^*(\gD)
		\qstq
		\choice{
				\Aa^* \gD + \fD = 0, \\
				\fD + \sum_k \fD_k=0.
		}
	}	
	The primal-dual relationships reads $\pi \in \partial G^*(-B^* \nu)$, and hence~\eqref{eq-primal-dual}. 
	Changing $-\fD_k/\bweight_k$ into $\fD_k$ gives the desired formula.\qquad
\end{proof}

Relevant examples of penalizations $J$ include the following:
\begin{itemize}
	\item In order to enforce some spread of the barycenter, one can use $\Aa=\Id$ and $J(\a) = \frac{\lambda}{2}\norm{\a}^2$, in which case $J^*(\fD) = \frac{1}{2 \lambda}\norm{\fD}^2$. In contrast to~\eqref{eq-dual-pbm}, the dual problem~\eqref{eq-regul-dual} is equivalent to an unconstraint smooth optimization. This problem can be solved using a simple Newton descent.  
	\item One can also enforce that the barycenter entries are smaller than some maximum value $\rho$ by setting $\Aa=\Id$ and $J = \iota_{\Cc}$, where $\Cc = \enscond{\a}{\normi{\a} \leq \rho}$. In this case, one has $J^*(\fD) = \rho \norm{\fD}_1$. The optimization~\eqref{eq-regul-dual} is equivalent an unconstrained nonsmooth optimization. Since the penalization is an $\ell^1$ norm, one solves it using first order proximal methods as detailed in section~\ref{sec-first-order-split} bellow.
	\item To force the barycenter to assume some fixed values $\a_I^0 \in \RR^{|I|}$ on a given set $I$ of indices, one can use $\Aa=\Id$ and $J=\iota_{\Cc}$, where $\Cc = \enscond{\a}{\a_I = \a_I^0}$, where $\a_I=(\a_i)_{i \in I}$. One then has $J^*(\fD) = \dotp{\fD_I}{\a_I^0} + \iota_{\{0\}}(\fD_{I^c})$.
	\item To force the barycenter to have some smoothness, one can select $\Aa$ to be a spacial derivative operator (for instance, a gradient approximated on some grid or mesh) and $J$ to be a norm such as an $\ell^2$ norm (to ensure uniform smoothness) or an $\ell^1$ norm (to ensure piecewise regularity). We explore this idea in Section~\ref{sec-tv-regul}.  
\end{itemize}

%%%%%%%%%%%%%%%%%%%%%%%%%%%%%%%%%%%%%%%%%%%%%%%%%%%%
%SEC 4.2
\subsection{Resolution using first order proximal splitting}
\label{sec-first-order-split}

Assuming without loss of generality that $\bweight_N \neq 0$ (otherwise one simply needs to permute the ordering of the input densities), one can note that it is possible to remove $\fD_N$ from~\eqref{eq-regul-dual} by imposing
\eq{
	\fD_N \eqdef - \frac{\Aa^* \gD}{\bweight_N} - \sum_{i=1}^{N-1} \frac{\bweight_k}{\bweight_N}  \fD_k, 
} 
and denoting $x \eqdef ((\fD_k)_{k=1}^{N-1}, \gD)$, one can consider the following optimization problem without the $H$ constraint
\eql{\label{eq-primal-unconstr}
	\umin{x} F(x) + G(x)
	\qwhereq
	\choice{
	F(x) \eqdef \sum_{k=1}^{N-1} \bweight_k \FD_{\b_k}^*(\fD_k) 
		+ \bweight_N \FD_{\b_N}^*\pa{  \fD_N(x)}, \\
	G(x) \eqdef J^*(\gD).
	}
}

We assume that one is able to compute the proximal operator of $J^*$
\eql{\label{eq-proximal-map}
	\Prox_{\tau J^*}(\gD) \eqdef \uargmin{\gD'} \frac{1}{2}\norm{\gD-\gD'}^2 + \tau J^*(\gD').
}
It is, for instance, an orthogonal projector on a convex set $C$ when $J^*=\iota_{C}$ is the indicator of $C$. 
One can easily compute  this projection, for instance, when $J$ is the $\ell^2$ or the $\ell^1$ norm (see Section~\ref{sec-tv-regul}). 
We refer to~\cite{BauschkeCombettes11} for more background on proximal operators.

The proximal operator of $G$ is then simply
\eq{
	\foralls x=((\fD_k)_{k=1}^{N-1}, \gD), \quad \Prox_{\tau G}(x) = ( (\fD_k)_{k=1}^{N-1}, \Prox_{\tau J^*}(\gD) ). 
}
Note also that the function $F$ is smooth with a Lipschitz gradient and that
\eq{
	\nabla F( (\fD_k)_{k=1}^{N-1}, \gD ) = 
	\pa{  
		(
			\bweight_k ( \nabla \FD_{\b_k}^*(\fD_k) 
			-
			\nabla \FD_{\b_N}^*(\fD_N)  )
		)_{k=1}^{N-1}, 		
		-\Aa \nabla \FD_{\b_N}^*(\fD_N)}.
}

The simplest algorithm to solve~\eqref{eq-primal-unconstr} is the forward-backward algorithm, whose iteration reads
\eql{\label{eq-algo-fb}
	\IIT{x} = \Prox_{\tau J^*}\pa{ \IT{x} - \tau \nabla F(\IT{x}) }.
}
If $0 < \tau<2/L$ where $L$ is the Lipschitz constant of $\nabla F$, then $\IT{x}$ converge to a solution of~\eqref{eq-primal-unconstr}; see~\cite{BauschkeCombettes11} and the references therein. In order to accelerate the convergence of the method, one can use accelerated schemes such as FISTA's algorithm~\cite{beck-fista}.

%%%%%%%%%%%%%%%%%%%%%%%%%%%%%%%%%%%%%%%%%%%%%%%%%%%
%SEC 4.3
\subsection{Total variation regularization}
\label{sec-tv-regul}

A typical example of regularization to enforce some geometrical regularity in the barycenter is the total variation regularization on a grid in $\RR^d$ (e.g., $d=2$ for images). It is obtained by considering
\eql{\label{eq-exmp-tv}
	\Aa \a \eqdef \nabla \a = ( \nabla_i \a )_i
	\qandq
	J(\fD) \eqdef \lambda \sum_i \norm{\fD_i}_{\beta}, 
}
where $\nabla_i \a \in \RR^d$ is a finite difference approximation of the gradient at a point indexed by $i$, and $\lambda \geq 0$ is the regularization strength.
When using the $\ell^2$ norm to measure the gradient amplitude, i.e., $\beta=2$, one obtains the so-called isotropic total variation, which tends to round corners and essentially penalizes the length of the level sets of the barycenter, possibly merging clusters together. 
When using instead the $\ell^1$ norm, i.e., $\beta=1$, one obtains the so-called anisotropic total variation, which penalizes independently horizontal and vertical derivative, thus favoring the emergence of axis-aligned edges and giving a ``crystalline'' look to the barycenters. We refer, for instance, to~\cite{2015-Caselles-tv} for a study of the effect of TV regularization on the shapes of levelsets using isotropic and crystalline total variations. 

In this case, it is possible to compute in closed form the proximal operator~\eqref{eq-proximal-map}. Indeed, one has $J^* = \iota_{\norm{\cdot}_{\beta^*} \leq \lambda}$, where $\beta^*$ is the conjugate exponent $1/\beta+1/\beta^*=1$. One can compute explicitly the proximal operator in the case $\beta \in \{1,2\}$ since they correspond to orthogonal projectors on $\ell^{\beta^*}$ balls
\eq{
	\Prox_{\tau J^*}(\gD)_i = 
	\choice{
		\min(\max(\gD_i,-\lambda), \lambda) \qifq \beta=1,  \\
		\gD_i \frac{\lambda}{ \max(\norm{\gD_i},1) }  \qifq \beta=2.
	}
}

%%%%%%%%%%%%%%%%%%%%%%%%%%%%%%%%%%%%%%%%%%%%%%%%%%%
%SEC 4.4
\subsection{Barycenters of images}
\label{sec-tv-images}

We start by computing barycenters of a small number of two-dimensional (2-D) images that are discretized on an uniform rectangular grid of $n=256 \times 256$ pixels $(z_i)_{i=1}^N$ of $[0,1]^2$. 
The entropic regularization parameter is set to $\epsilon = 1/n$.
We use either the isotropic ($\beta=2$) or anisotropic ($\beta=1$) total variation presented above, where $\nabla$ is defined using standard forward finite differences along each axis and using Neumann boundary conditions. The metric is the usual squared Euclidean metric
\eql{\label{eq-squared-eucl-metric}
	\foralls (i,j) \in \{1,\ldots,n\}^2, \quad
	\C_{i,j} = \norm{z_i-z_j}^2. 
}
The Gibbs kernel $\K=e^{-\C/\epsilon}$ is a filtering with a Gaussian kernel that can be applied efficiently to histograms in nearly linear time; see~\cite{2015-solomon-siggraph} for more details about convolutional kernels. 

%fig 4
\begin{figure}%[h!]
\iffalse
	\centering	
	\myTable{aniso}{0}{$\lambda=0$}	 
	\myTable{iso}{1000}{Isotropic, $\lambda=100$}	 
	\myTable{aniso}{500}{Anisotropic $\lambda=500$}	 
	\myTable{aniso}{2000}{Anisotropic $\lambda=2000$}	 
\fi
\centerline{\includegraphics{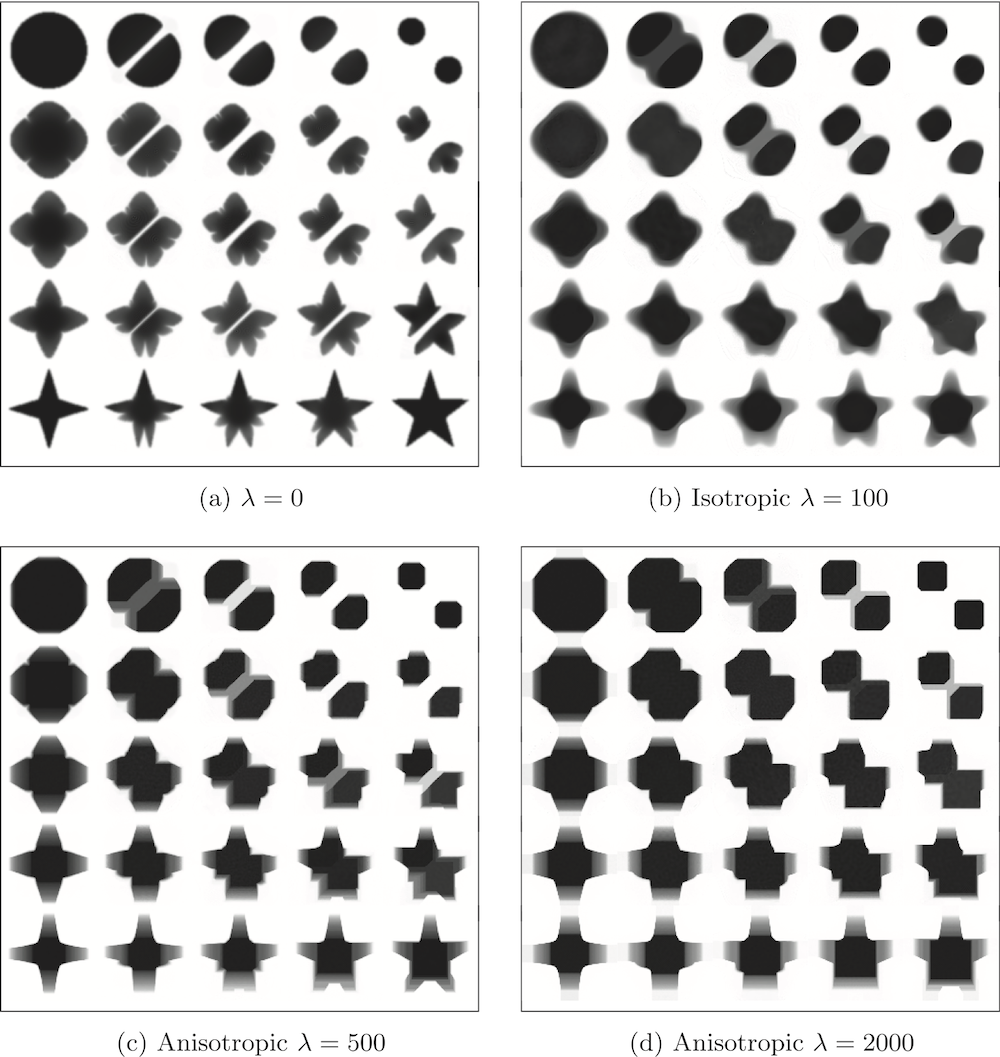}}
	\caption{Examples of isotropic and anisotropic TV regularization for the computation of barycenters between four input densities. 
		The weights $(\bweight_k)_{k=1}^N$ are bilinear interpolation weights, so that it is, for instance, $\bweight = (1,0,0,0)$ on the 
		top left corner and $(0,0,0,1)$ on the bottom right corner.}  \label{fig-images-bary}
\end{figure}

Figure~\ref{fig-images-bary} shows examples of barycenters of $N=4$ input histograms computed by solving~\eqref{eq-regul-primal} using the projected gradient descent method~\eqref{eq-algo-fb}. The input histograms represent 2-D shapes and are uniform (constant) distributions inside the support of the shapes. Note that in general the barycenters are not shapes, i.e., they are not uniform distributions, but this method can nevertheless be used to define meaningful averaging of shapes as exposed in~\cite{2015-solomon-siggraph}. 
Figure~\ref{fig-images-bary} compares the effects of $\beta \in \{1,2\}$, and one can clearly see how the isotropic total variation ($\beta=2$) rounds the corners of the input densities, while the anisotropic version ($\beta=1$) favors horizontal/vertical edges.  

%fig 5
\begin{figure}%[h!]
\iffalse
	\centering	
	\begin{tabular}{@{}c@{}c@{}c@{}c@{}c@{}c@{}c@{}c@{}}
		\myPicL{iso}{0} &
		\myPicL{iso}{200} &
		\myPicL{iso}{400} &
		\myPicL{iso}{600} &
		\myPicL{iso}{800} &
		\myPicL{iso}{1000} &
		\myPicL{iso}{2000} &
		\myPicL{iso}{3000} \\
		\myPicL{aniso}{0} &
		\myPicL{aniso}{200} &
		\myPicL{aniso}{400} &
		\myPicL{aniso}{600} &
		\myPicL{aniso}{800} &
		\myPicL{aniso}{1000} &
		\myPicL{aniso}{2000} &
		\myPicL{aniso}{3000} \\
		$\lambda=0$ &
		$\lambda=20$ &
		$\lambda=40$ &
		$\lambda=60$ &
		$\lambda=80$ &
		$\lambda=100$ &
		$\lambda=200$ &
		$\lambda=300$ 
	\end{tabular}	 
\fi
\centerline{\includegraphics{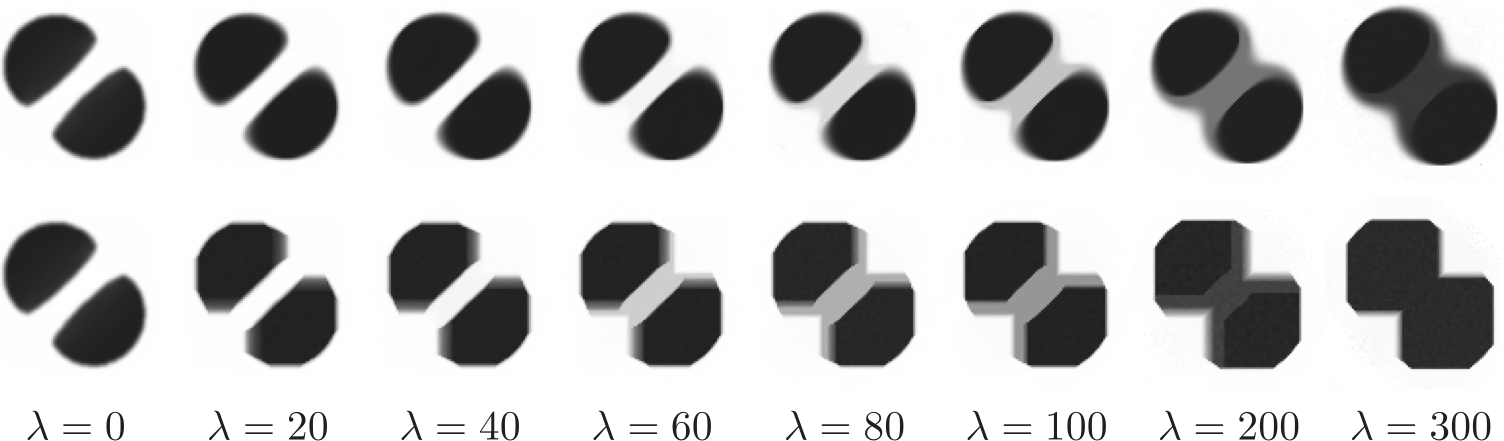}}
	\caption{Influence of $\lambda$ parameter for the iso-barycenter (i.e., $\bweight = (1/2,1/2)$) between two input densities. (They are the upper left and upper right corner of the $\lambda=0$ case in Figure~{\rm\ref{fig-images-bary}}.) 
		\textit{Top row:} isotropic total variation $(\beta=2)$.
		\textit{Bottom row:} anisotropic total variation $(\beta=1)$.} \label{fig-images-influ-lambda}
\end{figure}

Figure~\ref{fig-images-influ-lambda} shows the influence of the regularization strength $\lambda$ to compute the iso-barycenter of $N=2$ shapes. This highlights the fact that this total variation regularization has the tendency to group together small clusters, which might be beneficial for some applications, as illustrated in Section~\ref{sec-tv-meg} 
on magnetoencephalography (MEG) data denoising.

%%%%%%%%%%%%%%%%%%%%%%%%%%%%%%%%%%%%%%%%%%%%%%%%%%%
%SEC 4.5
\subsection{Barycenters of MEG data}
\label{sec-tv-meg}

We applied our method to an  MEG dataset. In this setup, the brain activity of a subject is recorded  (Elekta Neuromag, 306 sensors of which 204 planar gradiometers and 102 magnetometers, sampling frequency 1000Hz) while the subject reacted to the presentation of a target stimulus by pressing either the left or the right button.

Data is preprocessed applying signal space separation correction, interpolation of noisy sensors, and realignment of data into a subject-specific head position (MaxFilter, Elekta Neuromag). The signal was then filtered (low pass 40HZ), and artifacts such as blinks and heartbeats were 
removed thanks to signal-space projection using the Brainstorm software.\footnote{See \href{http://neuroimage.usc.edu/brainstorm}{http://neuroimage.usc.edu/brainstorm}.} The samples we used for our barycenter computations are an average of the norm of the two gradiometers for each channel from stimulation onto 50ms, and the classes were the left or right button.

This results in two classes of recordings, one for each pressed button. We aim at computing a representative activity map for each class using Wasserstein barycenters. For each class we have $N=33$ recordings $(\b_k)_{k=1}^N$ each having $n=66$ samples located on the vertices of an hexahedral mesh of a hemisphere (corresponding to a MEG recording helmet). These recorded values are positive by construction, and we rescale them linearly to impose $\b_k \in \Si_n$.  Figure~\ref{fig-meg}, top row, shows some samples from this dataset, displayed using interpolated colors as well as iso-level curves. The black dots represent the position $(z_i)_{i=1}^n$ of the electrodes on the half-sphere of the helmet, flattened on a 2-D disk. 

%fig 6
\begin{figure}%[h!]
\iffalse
	\centering	
	\begin{tabular}{@{}c@{}c@{}c@{}c@{}|@{}c@{}c@{}c@{}c@{}}
		& Class 1 & & & %
		& Class 2 & & \\ % 
		\myPicMEGin{1}{1} &
		\myPicMEGin{1}{2} &
		\myPicMEGin{1}{3} &
		\myPicMEGmean{1} & 
		%%%%
		\myPicMEGin{2}{1} &
		\myPicMEGin{2}{2} &
		\myPicMEGin{2}{3} &
		\myPicMEGmean{2} \\
		Sample 1 & 
		Sample 2 &
		Sample 3 &
		Mean &
		Sample 1 & 
		Sample 2 &
		Sample 3 &
		Mean  \\
		\myPicMEGbar{1}{0} &
		\myPicMEGbar{1}{2} &
		\myPicMEGbar{1}{4} &
		\myPicMEGbar{1}{8} &
		%%%
		\myPicMEGbar{2}{0} &
		\myPicMEGbar{2}{2} &
		\myPicMEGbar{2}{4} &
		\myPicMEGbar{2}{8} \\
		$\lambda=0$ &
		$\lambda=2$ &
		$\lambda=4$ &
		$\lambda=8$ &		
		%%%
		$\lambda=0$ &
		$\lambda=2$ &
		$\lambda=4$ &
		$\lambda=8$ 
	\end{tabular}	 
\fi
\centerline{\includegraphics[width=\linewidth]{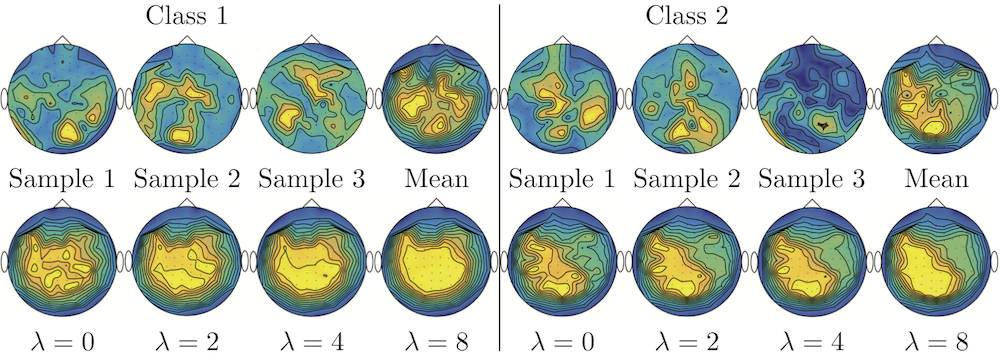}}
	\caption{Barycenter computation on MEG data. The left/right panels show, respectively, the first and the second class, 
		corresponding to recordings where the subject is asked to push the left or the right button.
		\textit{Top row:} examples of input histograms $\b_k$ for each class, as well as the 
		$\ell^2$ mean $N^{-1} \sum_k \b_k$.
		\textit{Bottom row:} computed TV-regularized barycenter for different values of $\lambda$
		($\lambda=0$ corresponding to no regularization).} \label{fig-meg}
\end{figure}

We computed TV-regularized barycenters independently for each class by solving~\eqref{eq-regul-primal} with the TV regularization using the projected gradient descent method~\eqref{eq-algo-fb}. We used a squared Euclidean metric~\eqref{eq-squared-eucl-metric} on the flattened hemisphere.
Since the data is defined on an irregular graph, instead of~\eqref{eq-exmp-tv}, we use a graph-based discrete gradient. We denote $( (i,j) )_{(i,j) \in \Gg}$ the graph which connects neighboring electrodes. The gradient operator on the graph is
\eq{
	\foralls \a \in \RR^n, \quad
	\Aa \a \eqdef ( \a_i - \a_j )_{ (i,j) \in \Gg } \in \RR^{|\Gg|}.
}
The total variation on this graph is then obtained by using $J = \lambda \norm{\cdot}_1$, the $\ell^1$ norm, i.e., we use $\beta=1$ in~\eqref{eq-exmp-tv}. 

Figure~\ref{fig-meg} compares the naive $\ell^2$ barycenters (i.e., the usual mean), barycenters obtained without regularization (i.e., $\lambda=0$), and barycenters computed with an increasing regularization strength $\lambda$. The input histograms $(\a_k)_k$ being very noisy, the use of regularization is important to make the area of significant activity emerge from the noise. The use of a TV regularization helps to keep a sharp transition between active and nonactive regions.

%%%%%%%%%%%%%%%%%%%%%%%%%%%%%%%%%%%%%%%%%%%%%%%%%%%%
%SEC 4.6
\subsection{Gradient flow}
\label{sec-grad-flows}

Instead of computing barycenters, we now use our regularization to define time-evolutions, which are defined through a so-called discrete gradient flow. 

Starting from an initial histogram $\a_0 \in \Si_n$, we define iteratively 
\eql{\label{eq-gradflow-def}
	\a_{k+1} \eqdef \uargmin{\a \in \Si_N} \FD_{\a_k}(\a) + \tau f(\a).
}
This means that one seeks a new iterate at (discrete time) $k+1$ that is both close (according to the entropy-regularized Wasserstein distance) to $\a_k$ and minimizes the functional $f$. In the following, we consider the gradient flow of regularization functionals as considered before, i.e., that are of the form $\tau f = J \circ \Aa$. Problem~\eqref{eq-gradflow-def} is thus a special case of~\eqref{eq-regul-primal} with $N=1$.

Letting $(k,\tau) \rightarrow (+\infty,0)$ with $t=k\tau$, one can informally think of $\a_k$ as a discretization of a time evolution evaluated at time $t$. This method is a general scheme presented in much detail in the monograph~\cite{ambrosio2006gradient}. The use of an implicit time-stepping~\eqref{eq-gradflow-def} allows one to define time evolutions to minimize functionals that are not necessarily smooth, and this is exactly the case of the total variation semi-norm (since $J$ is not differentiable). 
The use of gradient flows in the context of the Wasserstein fidelity to the previous iterate has been introduced initially in the seminal paper~\cite{jordan1998variational}. When $f$ is the entropy functional, this paper proves that the continuous flow defined by the limit $k \rightarrow+\infty$ and $\tau \rightarrow 0$ is a heat equation. Numerous theoretical papers have shown how to recover many existing nonlinear PDEs by considering the appropriate functional $f$; see, for instance,~\cite{otto2001geometry,GianazzaARMA}.

The numerical method we consider in this article is the one introduced in~\cite{Peyre-JKO}, which makes use of the entropic smoothing of the Wasserstein distance. It is not the scope of the present paper to discuss the problem of approximating gradient flows and the underlying limit nonlinear PDEs, and we refer to~\cite{Peyre-JKO} for an overview of the vast literature on this topic. A major bottleneck of the Sinkhorn-type algorithm developed in~\cite{Peyre-JKO} is that it uses a primal optimization scheme (Dykstra's algorithm) that necessitates the computation of the proximal operator of $f$ according to the Kullback--Leibler divergence. Only relatively simple functionals (basically separable functionals such as the entropy) can thus be treated by this approach. In contrast, our dual method can cope with a much larger set of functions, and in particular those of the form $f = J \circ \Aa$, i.e., obtained by pre-composition with a linear operator.

%fig 7
\begin{figure}%[h!]
\iffalse
	\centering	
	\begin{tabular}{@{}c@{\hspace{1mm}}c@{\hspace{1mm}}c@{\hspace{1mm}}c@{\hspace{1mm}}c@{\hspace{1mm}}c@{}}
		\myPicGF{randdisks}{0} &
		\myPicGF{randdisks}{2} &
		\myPicGF{randdisks}{5} &
		\myPicGF{randdisks}{10} &
		\myPicGF{randdisks}{15} &
		\myPicGF{randdisks}{20}  \\
		\myPicGF{hibiscus}{0} &
		\myPicGF{hibiscus}{2} &
		\myPicGF{hibiscus}{5} &
		\myPicGF{hibiscus}{10} &
		\myPicGF{hibiscus}{15} &
		\myPicGF{hibiscus}{20}  \\
		$t=0$ &
		$t=20$ &
		$t=40$ &
		$t=60$ &
		$t=80$ &
		$t=100$ 
	\end{tabular}
\fi
\centerline{\includegraphics[width=\linewidth]{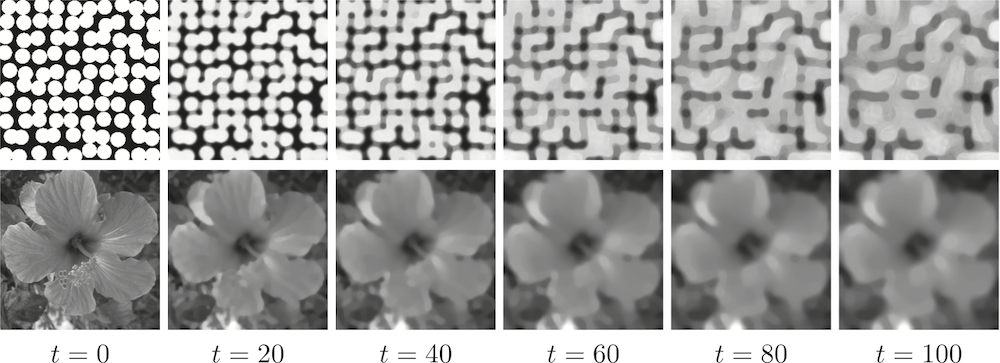}}	 
	\caption{Examples of gradient flows~\eqref{eq-gradflow-def} at various times $t \eqdef k\tau$.} \label{fig-gradient-flow}
\end{figure}

Figure~\ref{fig-gradient-flow} shows examples of gradient flows computed for the isotropic total variation $f(\a)=\norm{\nabla \a}_1$ as defined in~\eqref{eq-exmp-tv}. 
We use the discretization setup considered in Section~\ref{sec-tv-images}, with the same value of $\epsilon$, and a time-step $\tau = 1/10$. 
This is exactly the regularization flow considered by~\cite{Burger-JKO}, which is studied theoretically in~\cite{carlier2017total}. This paper defines formally the highly nonlinear fourth order PDE corresponding to the limit flow. This is, however, not a ``true'' PDE since the initial TV functional is non-smooth, and derivatives should be understood in a weak sense, as limit of an implicit discrete time-stepping. While the algorithm proposed in~\cite{Burger-JKO} uses the usual (unregularized) Wasserstein distance, the use of a regularized transport allows us to deal with problems of larger sizes, with a faster numerical scheme. The price to pay is an additional blurring introduced by the  entropic smoothing, but this is acceptable for applications to denoising in imaging. Figure~\ref{fig-gradient-flow} illustrates the behavior of this TV regularization flow, which has the tendency to group together clusters of mass and performs some kind of progressive ``percolation'' over the whole image. 

% !TEX root = ../SIGEST-SemiDual.tex

%\input sections/sec-conclusion
\section*{Conclusion}

In this paper, we introduced a dual framework for the resolution of certain variational problems involving Wasserstein distances. The key contribution is that the dual functional is smooth and that its gradient can be computed in closed form and involves only multiplications with a Gibbs kernel.
We illustrate this approach with applications to several problems revolving around the idea of Wasserstein barycenters. This method is particularly advantageous for the computation of regularized barycenters, since precomposition by a linear operator  (such as discrete gradient on images or graphs) of functionals is simple to handle. 
Our numerical finding is that entropic smoothing is crucial to stabilize the computation of barycenters and to obtain fast numerical schemes. 
Further regularization using, for instance, a total variation is also beneficial and can be used in the framework of gradient flows.

% one that involves computing Wasserstein barycenters and another that involves learning a dictionary and weights with a Wasserstein fit. Our approach has several attractive qualities: \emph{(i)} our entropic regularization ensures the unicity of the optimal solution in the simple WBP problem and facilitates the computation of each of the convex sub-problems considered in dictionary learning; \emph{(ii)} we observe that solutions obtained with this regularization exhibit a level of smoothness which is comparable to that of the original measures. This property can be desirable in some cases. \emph{(iii)} our approach can be initialized very efficiently thanks to a simple rule that is optimal in the simplified case where all original measures are dirac masses. \emph{(iv)} using Fenchel duality, we show that Wasserstein variational problems can be carried out using closed form functions. We believe this class of approaches can be extended to more general tasks and can scale up to more demanding learning problems.

%%%%%%%%%%%%%%%%%%%%%%%%%%%%%%%%%%%%%%%%%%%%%%%%%%%
\section*{Acknowledgments}

We would like to thank Antoine Rolet, Nicolas Papadakis and Julien Rabin for stimulating discussions. 
We would like to thank Valentina Borghesani, Manuela Piazza et Marco Buiatti, for giving us access to the MEG data.
We would like to thank Fabian Pedregosa and the chaire ``\'Economie des Nouvelles Donn\'ees'' for the help in the preparation of the MEG data.

\bibliographystyle{plainnat}
\bibliography{refs}

\end{document}